\documentclass[twoside,11pt]{article}

\usepackage{jmlr2e}

\usepackage{algorithm}
\usepackage{algorithmic}
\usepackage{hyperref}

\usepackage{multirow}

\usepackage{amsmath,amssymb,dsfont}

\firstpageno{1}

\begin{document}

\title{Stochastic Smoothing for Nonsmooth Minimizations:\\ Accelerating SGD by Exploiting Structure}

\author{\name Hua Ouyang, Alexander Gray\hspace{3em} \email \{houyang, agray\}@cc.gatech.edu\\
\addr College of Computing\\
Georgia Institute of Technology
}

\editor{Leslie Pack Kaelbling}

\maketitle

\begin{abstract}
In this work we consider the stochastic minimization of nonsmooth convex loss functions, a central problem in machine learning. We propose a novel algorithm called \textsf{A}ccelerated \textsf{N}onsmooth \textsf{S}tochastic \textsf{G}radient \textsf{D}escent (\textsf{ANSGD}), which exploits the structure of common nonsmooth loss functions to achieve optimal convergence rates for a class of problems including SVMs. It is the first stochastic algorithm that can achieve the optimal $O(1/t)$ rate for minimizing nonsmooth loss functions (with strong convexity). The fast rates are confirmed by empirical comparisons, in which \textsf{ANSGD} significantly outperforms previous subgradient descent algorithms including SGD.
\end{abstract}
\footnotetext[1]{A short version of this paper appears in International Conference of Machine Learning (ICML) 2012.}

\section{Introduction}
\label{sec:intro}
Nonsmoothness is a central issue in machine learning computation, as many important methods minimize nonsmooth convex functions. For example, using the nonsmooth hinge loss yields sparse support vector machines; regressors can be made robust to outliers by using the nonsmooth absolute loss other than the squared loss; the $l1$-norm is widely used in sparse reconstructions. In spite of the attractive properties, nonsmooth functions are theoretically more difficult to optimize than smooth functions \cite{nemirovski83pcmeo}. In this paper we focus on minimizing nonsmooth functions where the functions are either stochastic (stochastic optimization), or learning samples are provided incrementally (online learning).

Smoothness and strong-convexity are typically certificates of the existence of fast global solvers. Nesterov's deterministic smoothing method \cite{nesterov05smnsf} deals with the difficulty of nonsmooth functions by approximating them with smooth functions, for which optimal methods \cite{nesterov04ilco} can be applied. It converges as $f(\mathbf{x}_t)-\min_{\mathbf{x}}f(\mathbf{x})\leq O(1/t)$ after $t$ iterations. If a nonsmooth function is strongly convex, this rate can be improved to $O(1/t^2)$ using the excessive gap technique \cite{nesterov05egtncm}.

In this paper, we extend Nesterov's smoothing method to the stochastic setting by proposing a stochastic smoothing method for nonsmooth functions. Combining this with a stochastic version of the optimal gradient descent method, we introduce and analyze a new algorithm named \textsf{A}ccelerated \textsf{N}onsmooth \textsf{S}tochastic \textsf{G}radient \textsf{D}escent (\textsf{ANSGD}), for a class of functions that include the popular ML methods of interest.

To our knowledge \textsf{ANSGD} is the first stochastic first-order algorithm that can achieve the optimal $O(1/t)$ rate for minimizing nonsmooth loss functions without Polyak's averaging \cite{polyak92asaa}. In comparison, the classic SGD converges in $O(\ln t/t)$ for nonsmooth strongly convex functions \cite{sss07pegasos}, and is usually not robust \cite{nemirovski09rsaasp}. Even with Polyak's averaging \cite{bach11naasaaml,xu11tooplslasgd}, there are cases where SGD's convergence rate still can not be faster than $O(\ln t/t)$ \cite{shamir11mgdoscso}. Numerical experiments on real-world datasets also indicate that \textsf{ANSGD} converges much faster in comparing with these state-of-the-art algorithms.

A perturbation-based smoothing method is recently proposed for stochastic nonsmooth minimization \cite{duchi11rsso}. This work achieves similar iteration complexities as ours, in a parallel computation scenario. In serial settings, \textsf{ANSGD} enjoys better and optimal bounds.

%It is only with the accelerated method that this smoothing can be useful, otherwise, it achieves the same rate as stochastic subgradient descent.

%Utilizing the strongly convex in l1-regularized least squares
%Different bound from Lan's: our bound is min(/muT, /sqrt{T}), which makes it more robust w.r.t. small mu.

% why not using ACSA or SAGE? because L_t is stochastic
% differ from ACSA: rates in N
% differ from SAGE: in strongly convex case, not that easy to find a suitable theta

In machine learning, many problems can be cast as minimizing a composition of a loss function and a regularization term. Before proceeding to the algorithm, we first describe a different setting of ``composite minimizations'' that we will pursue in this paper, along with our notations and assumptions.
\subsection{A Different ``Composite Setting''}
In the classic \emph{black-box} setting of first-order stochastic algorithms \cite{nemirovski09rsaasp}, the structure of the objective function $\min_{\mathbf{x}}\{ f(\mathbf{x})=\mathbb{E}_{\boldsymbol{\xi}}f(\mathbf{x},\boldsymbol{\xi}):\boldsymbol{\xi}\sim P\}$ is unknown. In each iteration $t$, an algorithm can only access the first-order stochastic oracle and obtain a subgradient $f^{\prime}(\mathbf{x},\boldsymbol{\xi}_t)$. The basic assumption is that $f^{\prime}(\mathbf{x}) = \mathbb{E}_{\boldsymbol{\xi}}f^{\prime}(\mathbf{x},\boldsymbol{\xi})$ for any $\mathbf{x}$, where the random vector $\boldsymbol{\xi}$ is from a fixed distribution $P$.

The \emph{composite setting} (also known as \emph{splitting} \cite{lions79splitting}) is an extension of the black-box model. It was proposed to exploit the structure of objective functions. Driven by applications of sparse signal reconstruction, it has gained significant interest from different communities \cite{daubechies04ita,beck09fista,nesterov07cof}. Stochastic variants have also been proposed recently \cite{lan10omsco,lan11osaascsco1,duchi09fobos,hu09agmsool,xiao10damrsloo}. A stochastic composite function $\Phi(\mathbf{x}):=f(\mathbf{x})+g(\mathbf{x})$ is the sum of a smooth stochastic convex function $f(\mathbf{x})=\mathbb{E}_{\boldsymbol{\xi}}f(\mathbf{x},\boldsymbol{\xi})$ and a nonsmooth (but simple and deterministic) function $g()$. To minimize $\Phi$, previous work construct the following model iteratively:
\begin{equation}\label{eq:comp_set}
\langle \nabla f(\mathbf{x}_t,\boldsymbol{\xi}_t),\mathbf{x}-\mathbf{x}_t\rangle+\frac{1}{\eta_{t}}D(\mathbf{x},\mathbf{x}_t)+g(\mathbf{x}),
\end{equation}
where $\nabla f(\mathbf{x}_t,\boldsymbol{\xi}_t)$ is a gradient, $D(\cdot,\cdot)$ is a proximal function (typically a Bregman divergence) and $\eta_t$ is a stepsize.

A successful application of the composite idea typically relies on the assumption that model (\ref{eq:comp_set}) is easy to minimize. If $g()$ is very simple, e.g. $\|\mathbf{x}\|_1$ or the nuclear norm, it is straightforward to obtain the minimum in analytic forms. However, this assumption does not hold for many other applications in machine learning, where many loss functions (not the regularization term, here the nonsmooth $g()$ becomes the nonsmooth loss function) are nonsmooth, and do not enjoy separability properties \cite{wright09srsa}. This includes important examples such as hinge loss, absolute loss, and $\epsilon$-insensitive loss.

In this paper, we tackle this problem by studying a new stochastic composite setting: $\min_{\mathbf{x}}\Phi(\mathbf{x})=f(\mathbf{x})+g(\mathbf{x})$, where loss function $f()$ is convex and nonsmooth, while $g()$ is convex and $L_g$-Lipschitz smooth:
\begin{equation}\label{eq:g_lipschitz}
g(\mathbf{x})\leq g(\mathbf{y}) + \langle \nabla g(\mathbf{y}),\mathbf{x}-\mathbf{y} \rangle + \frac{L_g}{2}\|\mathbf{x}-\mathbf{y}\|^2.
\end{equation}
For clarity, in this paper we focus on unconstrained minimizations. Without loss of generality, we assume that both $f()$ and $g()$ are stochastic: $f(\mathbf{x})=\mathbb{E}_{\boldsymbol{\xi}}f(\mathbf{x},\boldsymbol{\xi})$ and $g(\mathbf{x})=\mathbb{E}_{\boldsymbol{\xi}}g(\mathbf{x},\boldsymbol{\xi})$, where $\boldsymbol{\xi}$ has distribution $P$. If either one is deterministic, its $\boldsymbol{\xi}$ is then dropped. To make our algorithm and analysis more general, we assume that $g()$ is $\mu$-strongly convex: $\forall \mathbf{x},\mathbf{y}$,
\begin{equation}\label{eq:g_mu_str}
g(\mathbf{x})\geq g(\mathbf{y}) + \langle \nabla g(\mathbf{y}),\mathbf{x}-\mathbf{y} \rangle + \frac{\mu}{2}\|\mathbf{x}-\mathbf{y}\|^2.
\end{equation}
If it is not strongly convex, one can simply take $\mu=0$.

The main idea of our algorithm again stems from exploiting the structures of $f()$ and $g()$. In Section \ref{sec:approach} we propose to form a smooth stochastic approximation of $f()$, such that the optimal methods \cite{nesterov04ilco} can be applied to attain optimal convergence rates. The convergence of our proposed algorithm is analyzed in Section \ref{sec:analysis}, and a batch-to-online conversion is also proposed. Two popular machine learning problems are chosen as our examples in Section \ref{sec:examples}, and numerical evaluations are presented in Section \ref{sec:exp}. All proofs in this paper are provided in the appendix. %\footnote{http://full.length.paper.with.appendix}.

\section{Approach}\label{sec:approach}
\subsection{Stochastic Smoothing Method}
An important breakthrough in nonsmooth minimization was made by Nesterov in a series of works \cite{nesterov05smnsf,nesterov05egtncm,nesterov07staso}. By exploiting function structures, Nesterov shows that in many applications, minimizing a well-structured nonsmooth function $f(\mathbf{x})$ can be formulated as an equivalent saddle-point form
\begin{equation}\label{eq:equi_saddle}
\min_{\mathbf{x}\in\mathcal{X}}f(\mathbf{x}) = \min_{\mathbf{x}\in\mathcal{X}}\max_{\mathbf{u}\in\mathcal{U}}\bigg[ \langle A\mathbf{x},\mathbf{u}\rangle - Q(\mathbf{u}) \bigg],
\end{equation}
where $\mathbf{u}\in\mathbb{R}^m$, $\mathcal{U}\subseteq \mathbb{R}^m$ is a convex set, $A$ is a linear operator mapping $\mathbb{R}^D\rightarrow \mathbb{R}^m$ and $Q(\mathbf{u})$ is a continuous convex function. Inserting a non-negative $\zeta$-strongly convex function $\omega(\mathbf{u})$ in (\ref{eq:equi_saddle}) one obtains a smooth approximation of the original nonsmooth function
\begin{equation}
\hat{f}(\mathbf{x},\gamma) := \max_{\mathbf{u}\in\mathcal{U}}\bigg[ \langle A\mathbf{x},\mathbf{u}\rangle - Q(\mathbf{u}) - \gamma \omega(\mathbf{u}) \bigg],
\end{equation}
where $\gamma>0$ is a fixed \emph{smoothness parameter} which is crucial in the convergence analysis. The key property of this approximation is:
\begin{lemma}\label{lm:approx_smooth}\cite{nesterov05smnsf}(Theorem 1)
Function $\hat{f}(\mathbf{x},\gamma)$ is convex and continuously differentiable, and its gradient is Lipschitz continuous with constant $L_{\hat{f}}:=\frac{\|A\|^2}{\gamma\zeta}$, where
\begin{equation}\label{eq:l12norm}
\|A\|:=\max_{\mathbf{x},\mathbf{u}}\{\langle A\mathbf{x},\mathbf{u} \rangle:\|\mathbf{x}\|=1,\|\mathbf{u}\|=1 \}.
\end{equation}
\end{lemma}

Nesterov's smoothing method was originally proposed for deterministic optimization. A major drawback of this method is that the number of iterations $N$ must be known beforehand, such that the algorithm can set a proper smoothness parameter $\gamma = O\big(\frac{2\|A\|}{N+1}\big)$ to ensure convergence. This makes it unsuitable for algorithms that runs forever, or whose number of iterations is not known. Following his work we propose to extend this smoothing method to stochastic optimization. Our stochastic smoothing differs from the deterministic one in the operator $A$ and smoothness parameter $\gamma$, where both will be time-varying.

We assume that the nonsmooth part $f(\mathbf{x},\boldsymbol{\xi})$ of the stochastic composite function $\Phi()$ is well structured, i.e. for a specific realization $\boldsymbol{\xi}_t$, it has an equivalent form like the max function in (\ref{eq:equi_saddle}):
\begin{equation}\label{eq:def_f}
f(\mathbf{x},\boldsymbol{\xi}_t) = \max_{\mathbf{u}\in\mathcal{U}}\bigg[ \langle A_{\boldsymbol{\xi}_t}\mathbf{x}, \mathbf{u} \rangle - Q(\mathbf{u}) \bigg],
\end{equation}
where $A_{\boldsymbol{\xi}_t}$ is a stochastic linear operator associated with $\boldsymbol{\xi}_t$. We construct a smooth approximation of this function as:
\begin{equation}\label{eq:def_hat_f}
\hat{f}(\mathbf{x},\boldsymbol{\xi}_t,\gamma_t) := \max_{\mathbf{u}\in\mathcal{U}}\bigg[ \langle A_{\boldsymbol{\xi}_t}\mathbf{x},\mathbf{u}\rangle - Q(\mathbf{u}) - \gamma_t \omega(\mathbf{u}) \bigg],
\end{equation}
where $\gamma_t$ is a time-varying smoothness parameter only associated with iteration index $t$, and is independent of $\boldsymbol{\xi}_t$. Function $\omega()$ is non-negative and $\zeta$-strongly convex. Due to Lemma \ref{lm:approx_smooth}, $\hat{f}(\mathbf{x},\boldsymbol{\xi}_t,\gamma_t)$ is $\frac{\|A_{\boldsymbol{\xi}_t}\|^2}{\gamma_t\zeta}$-Lipschitz smooth.
It follows that
\begin{lemma}\label{lm:exp_lips}
$\forall \mathbf{x},\mathbf{y},t$,
$\mathbb{E}_{\boldsymbol{\xi}}\hat{f}(\mathbf{x},\boldsymbol{\xi},\gamma_t) \leq \mathbb{E}_{\boldsymbol{\xi}}\hat{f}(\mathbf{y},\boldsymbol{\xi},\gamma_t) + \mathbb{E}_{\boldsymbol{\xi}}\langle \nabla\hat{f}(\mathbf{y},\boldsymbol{\xi},\gamma_t) , \mathbf{x}-\mathbf{y}\rangle + \frac{\mathbb{E}_{\boldsymbol{\xi}}\|A_{\boldsymbol{\xi}}\|^2}{\gamma_t\zeta}\|\mathbf{x}-\mathbf{y}\|^2$.
\end{lemma}
We have the following observation about our composite objective $\Phi()$, which relates the reduction of the original and approximated function values.
\begin{lemma}\label{lm:agd_phi}
For any $\mathbf{x},\mathbf{x}_t,t$,
\begin{equation}
\begin{split}
\Phi(\mathbf{x}_t) - \Phi(\mathbf{x}) &\leq  \mathbb{E}_{\boldsymbol{\xi}}\left[\hat{f}(\mathbf{x}_t,\boldsymbol{\xi},\gamma_t) + g(\mathbf{x}_t,\boldsymbol{\xi}) \right] - \mathbb{E}_{\boldsymbol{\xi}}\left[\hat{f}(\mathbf{x},\boldsymbol{\xi},\gamma_t) + g(\mathbf{x},\boldsymbol{\xi}) \right] + \gamma_t D_{\mathcal{U}},
\end{split}
\end{equation}
where $D_{\mathcal{U}}:=\max_{\mathbf{u}\in\mathcal{U}}\omega(\mathbf{u})$.
\end{lemma}

\subsection{Accelerated Nonsmooth SGD (\textsf{ANSGD})}
We are now ready to present our algorithm \textsf{ANSGD} (Algorithm \ref{alg:ansgd}). This stochastic algorithm is obtained by applying Nesterov's optimal method to our smooth surrogate function, and thus has a similar form to that of his original deterministic method \cite{nesterov04ilco}(p.78). However, our convergence analysis is more straightforward, and does not rely on the concept of estimate sequences. Hence it is easier to identify proper series $\gamma_t, \eta_t, \alpha_t$ and $\theta_t$ that are crucial in achieving fast rates of convergence. These series will be determined in our main results (Thm.\ref{thm:result_general_convex} and \ref{thm:result_strongly_convex}).
%To simplify the notations, we use subscript $t$ to replace $\boldsymbol{\xi}_t$, i.e. $\hat{f}_t(\mathbf{x})=f(\mathbf{x},\boldsymbol{\xi}_t)$, $\nabla\hat{f}_t(\mathbf{x}) = \nabla\hat{f}(\mathbf{x},\boldsymbol{\xi}_t)$, and the same rule applies to $g()$.

\begin{algorithm}
\caption{\textsf{Accelerated Nonsmooth Stochastic Gradient Descent (ANSGD)}}
\label{alg:ansgd}
\begin{algorithmic}
\STATE INPUT: series $\gamma_t,\ \eta_t,\ \theta_t\geq 0$ and $0\leq \alpha_t\leq 1$;
\STATE OUTPUT: $\mathbf{x}_{t+1}$;
\STATE[0.] Initialize $\mathbf{x}_0$ and $\mathbf{v}_0$;
\FOR{$t=0,1,2,\ldots$}
  \STATE[1.] $\mathbf{y}_{t} \leftarrow \frac{(1-\alpha_t)(\mu+\theta_t)\mathbf{x}_t+\alpha_t\theta_t\mathbf{v}_t}{\mu(1-\alpha_t)+\theta_t}$
  \STATE[2.] $\hat{f}_{t+1}(\mathbf{x})\leftarrow \displaystyle\max_{\mathbf{u}\in\mathcal{U}} \bigg[ \langle A_{\boldsymbol{\xi}_{t+1}}\mathbf{x}, \mathbf{u}\rangle - Q(\mathbf{u}) - \gamma_{t+1} \omega(\mathbf{u})\bigg]$
  \STATE[3.] $\mathbf{x}_{t+1} \leftarrow \mathbf{y}_t - \eta_t \bigg[\nabla\hat{f}_{t+1}(\mathbf{y}_t)+ \nabla g_{t+1}(\mathbf{y}_t)\bigg]$
  \STATE[4.] $\mathbf{v}_{t+1} \leftarrow \frac{\theta_t \mathbf{v}_t + \mu\mathbf{y}_t - \left[\nabla\hat{f}_{t+1}(\mathbf{y}_t)+ \nabla g_{t+1}(\mathbf{y}_t)\right]}{\mu+\theta_t}$
\ENDFOR
\end{algorithmic}
\end{algorithm}
\section{Convergence Analysis}\label{sec:analysis}
To clarify our presentation, we use Table \ref{tab:notation} to list some notations that will be used throughout the paper.
\begin{table}[h!]\caption{Some notations.\label{tab:notation}}
\begin {center}
\setlength{\tabcolsep}{5pt} \begin{tabular}{c|c}
    \hline
    Symbol & Meaning\\
    \hline \hline
    $\hat{f}_t(\mathbf{x})$, $g_t(\mathbf{x})$ & $\hat{f}(\mathbf{x},\boldsymbol{\xi}_t,\gamma_t)$, $g(\mathbf{x},\boldsymbol{\xi}_t)$\\
    $\nabla\hat{f}_t(\mathbf{x})$, $\nabla g_t(\mathbf{x})$ & $\nabla\hat{f}(\mathbf{x},\boldsymbol{\xi}_t,\gamma_t)$, $\nabla g(\mathbf{x},\boldsymbol{\xi}_t)$\\
    $L_t$ & $L_g+\frac{\|A_{\boldsymbol{\xi}_t}\|^2}{\gamma_{t}\zeta}$ \\
    $\sigma_t(\mathbf{x})$ & $[\nabla \hat{f}_t(\mathbf{x})+\nabla g_t(\mathbf{x})] - \mathbb{E}_{\boldsymbol{\xi}_t}[\nabla \hat{f}_t(\mathbf{x})+\nabla g_t(\mathbf{x})]$ \\
    $\sigma^2$ & $\mathbb{E}\max_t\|\sigma_{t+1}(\mathbf{y}_t)\|^2$ \\
    $\Delta_t$ & $\mathbb{E}_{\boldsymbol{\xi}_t}\big[\hat{f}_t(\mathbf{x}_{t})+g_t(\mathbf{x}_{t})\big] - \mathbb{E}_{\boldsymbol{\xi}_t}\big[ \hat{f}_t(\mathbf{x})+g(\mathbf{x})\big]$ \\
    $\Gamma_{t+1}$ & $\langle \sigma_{t+1}(\mathbf{y}_t), \alpha_t\mathbf{x}+(1-\alpha_t)\mathbf{x}_t - \mathbf{y}_t\rangle$\\
    $D_t^2$ & $\frac{1}{2}\mathbb{E}\|\mathbf{x}-\mathbf{v}_t\|^2$ \\
    \hline
\end{tabular}
\end{center}
\end{table}

Our convergence rates are based on the following main lemma, which bounds the progressive reduction $\Delta_t$ of the smoothed function value. Actually Line 1, 3, and 4 of Alg.\ref{alg:ansgd} are also derived from the proof of this lemma.
%In the following analysis we denote the difference between a sample gradient and its expectation as $\sigma_t(\mathbf{x}) := \big[\nabla \hat{f}_t(\mathbf{x})+\nabla g_t(\mathbf{x})\big] - \big[\nabla \hat{f}(\mathbf{x})+\nabla g(\mathbf{x})\big]$.
%Denote $\Delta_t := \big[\hat{f}(\mathbf{x}_{t})+g(\mathbf{x}_{t})\big] - \big[ \hat{f}(\mathbf{x})+g(\mathbf{x})\big]$.
\begin{lemma}\label{lm:agd_lm1}
Let $\gamma_t$ be monotonically decreasing. Applying algorithm \textsf{ANSGD} to nonsmooth composite function $\Phi()$, we have $\forall \mathbf{x}$ and $\forall t\geq 0$,
\begin{equation}\label{eq:main_lm_result}
\begin{split}
&\Delta_{t+1} \leq (1-\alpha_t)\Delta_t +(1-\alpha_t)(\gamma_t-\gamma_{t+1})D_{\mathcal{U}} + \\
& \Gamma_{t+1}+ \frac{\alpha_t}{2}\bigg[ \theta_t\|\mathbf{x}-\mathbf{v}_t\|^2-(\mu+\theta_t)\|\mathbf{x}-\mathbf{v}_{t+1}\|^2 \bigg]+\\
&\eta_t p q+
\left[\frac{\alpha_t}{2(\mu+\theta_t)}+\frac{L_{t+1}}{2}\eta_t^2-\eta_t\right]q^2
\end{split}
\end{equation}
where $p:=\|\sigma_{t+1}(\mathbf{y}_t)\|$ and $q:=\|\nabla \hat{f}_{t+1}(\mathbf{y}_t)+\nabla g_{t+1}(\mathbf{y}_t)\|$.
%$L_t:=L_g+\frac{\|A_{\boldsymbol{\xi}_t}\|^2}{\gamma_{t}\zeta}$,
%and $\Gamma_{t+1}:=\langle \sigma_{t+1}(\mathbf{y}_t), \alpha_t\mathbf{x}+(1-\alpha_t)\mathbf{x}_t - \mathbf{y}_t\rangle$.
\end{lemma}
\subsection{How to Choose Stepsizes $\eta_t$}
In the RHS of (\ref{eq:main_lm_result}), nonnegative scalars $p,q\geq 0$ are data-dependent, and could be arbitrarily large. Hence we need to set proper stepsizes $\eta_t$ such that the last two terms in (\ref{eq:main_lm_result}) are non-positive. One might conjecture that: there exist a series $c_t\geq 0$ such that
\begin{equation}\label{eq:lm:select_eta}
\eta_t p q+
\left[\frac{\alpha_t}{2(\mu+\theta_t)}+\frac{L_{t+1}}{2}\eta_t^2-\eta_t\right]q^2 \leq c_t p^2.
\end{equation}
It is easy to verify that if we take $\eta_t = \frac{\alpha_t}{\mu+\theta_t}$
and any series $c_t \geq \frac{\alpha_t}{2(\mu+\theta_t-\alpha_t L_{t+1})}\geq 0$,
then (\ref{eq:lm:select_eta}) is satisfied. To retain a tight bound, we take
\begin{equation}\label{eq:ct}
c_t = \frac{\alpha_t}{2(\mu+\theta_t-\alpha_t L_{t+1})}.
\end{equation}
Taking expectation on both sides of (\ref{eq:main_lm_result}) and noticing that $\mathbb{E}_{\boldsymbol{\xi_{t+1}|\xi_{[t]}}}\Gamma_{t+1}=0$, $\mathbb{E}_{\boldsymbol{\xi}_{t+1}}c_t \leq \frac{\alpha_t}{2(\mu+\theta_t-\alpha_t\mathbb{E}_{\boldsymbol{\xi}_{t+1}}L_{t+1})}$ due to Jensen's inequality, we have
\begin{lemma}
$\forall \mathbf{x}$ and $\forall t\geq 0$,
\begin{equation}\label{eq:result_base}
\begin{split}
&\mathbb{E}\Delta_{t+1} \leq (1-\alpha_t)\mathbb{E}\Delta_t + \alpha_t\theta_t D_t^2-\alpha_t(\mu+\theta_t) D_{t+1}^2\\
& + \frac{\alpha_t}{2(\mu+\theta_t-\alpha_t\mathbb{E}L_{t+1})}\sigma^2 + (1-\alpha_t)(\gamma_t-\gamma_{t+1})D_{\mathcal{U}},
\end{split}
\end{equation}
\end{lemma}
%where we denote $D_t^2:=\frac{1}{2}\mathbb{E}\|\mathbf{x}-\mathbf{v}_t\|^2$ and $\sigma^2:=\mathbb{E}\max_t\|\sigma_{t+1}(\mathbf{y}_t)\|^2$.
The optimal convergence rates of our algorithm differs according to the fact of $\mu$ (positive or not). They are presented separately in the following two subsections, where the choices of $\gamma_t,\ \theta_t,\ \alpha_t$ will also be determined.
\subsection{Optimal Rates for Composite Minimizations when $\mu=0$}
When $\mu=0$, $g()$ is only convex and $L_g$-Lipschitz smooth, but not assumed to be strongly convex.
\begin{theorem}\label{thm:result_general_convex}
Take $\alpha_t = \frac{2}{t+2}$, $\gamma_{t+1} = \alpha_{t}$, $\theta_t = L_g\alpha_t+ \frac{\Omega}{\sqrt{\alpha_t}}+\frac{\mathbb{E}\|A_{\boldsymbol{\xi}}\|^2}{\zeta}$ and $\eta_t = \frac{\alpha_t}{\theta_t}$ in Alg.\ref{alg:ansgd}, where $\Omega$ is a constant. We have $\forall \mathbf{x}$ and $\forall t\geq 0$,
\begin{equation}\label{eq:thm1}
\mathbb{E}\left[\Phi(\mathbf{x}_{t+1}) - \Phi(\mathbf{x})\right] \leq \frac{4L_g D^2}{(t+2)^2} + \frac{2\mathbb{E}\|A_{\boldsymbol{\xi}}\|^2 D^2/\zeta + 4D_{\mathcal{U}}}{t+2} +\frac{\sqrt{2}(\Omega D^2+\sigma^2/\Omega)}{\sqrt{t+2}},
\end{equation}
where $D^2:=\max_i D_i^2$.
\end{theorem}
In this result, the variance bound is optimal up to a constant factor \cite{agarwal12itlb}. The dominating factor is still due to the stochasticity, but not affected by the nonsmoothness of $f()$. Taking the parameter $\Omega = \sigma/D$, this last term becomes $\frac{2\sqrt{2}D \sigma}{\sqrt{t+2}}$. This bound is better than that of stochastic gradient descent or stochastic dual averaging \cite{dekel10odopmb} for minimizing $L$-Lipschitz smooth functions, whose rate is $O\left(\frac{LD_0^2}{t}+\frac{D_0^2+\sigma^2}{\sqrt{t}}\right)$; without the smooth function $g()$, our bound is of the same order as it, keeping in mind that our rate is for nonsmooth minimizations. This fact underscores the potential of using stochastic optimal methods for nonsmooth functions.

The diminishing smoothness parameter $\gamma_t=\frac{2}{t+2}$ indicates that initially a smoother approximation is preferred, such that the solution does not change wildly due to the nonsmoothness and stochasticity. Eventually the approximated function should be closer and closer to the original nonsmooth function, such that the optimality can be reached. Some concrete examples are given in Fig.\ref{fig:hinge_and_abs}.

The $\mathbb{E}\|A_{\boldsymbol{\xi}}\|^2$ in our bound is a theoretical constant. In Sec.\ref{sec:examples} we demonstrate a sampling method, and it turns out to work quite well in estimating $\mathbb{E}\|A_{\boldsymbol{\xi}}\|^2$.

\subsection{Nearly Optimal Rates for Strongly Convex Minimizations}
When $\mu> 0$, $g()$ is strongly convex, and the convergence rate of \textsf{ANSGD} can be improved to $O(1/t)$.
\begin{theorem}\label{thm:result_strongly_convex}
Take $\alpha_t = \frac{2}{t+1}$, $\gamma_{t+1} = \alpha_{t}$, $\theta_t = L_g\alpha_t+ \frac{\mu}{2\alpha_t}+\frac{\mathbb{E}\|A_{\boldsymbol{\xi}}\|^2}{\zeta} - \mu$ and $\eta_t = \frac{\alpha_t}{\mu+\theta_t}$ in Alg.\ref{alg:ansgd}. Denote
\begin{equation}\label{eq:thm2_def_c}
C:= \max\left\{\frac{4\mathbb{E}\|A_{\boldsymbol{\xi}}\|^2}{\zeta\mu} , 2\left( \frac{L_g}{\mu} \right)^{1/3} \right\}.
\end{equation}
We have $\forall \mathbf{x}$ and $\forall t\geq 0$,
\begin{equation}\label{eq:thm2}
\mathbb{E}\left[\Phi(\mathbf{x}_{t+1}) - \Phi(\mathbf{x})\right] \leq \frac{6.58 L_g \tilde{D}^2}{t(t+1)} +\mathcal{B} +\frac{4D_{\mathcal{U}}}{t+1}+\frac{\sigma^2}{\mu(t+1)},
\end{equation}
where
\begin{equation}
\mathcal{B}:=\begin{cases}
\frac{2\mathbb{E}\|A_{\boldsymbol{\xi}}\|^2 \tilde{D}^2/\zeta}{t+1} & \text{if } 0\leq t<C,\\
\frac{2(C-2)\mathbb{E}\|A_{\boldsymbol{\xi}}\|^2 \tilde{D}^2/\zeta }{t(t+1)} & \text{if } t\geq C,
\end{cases}
\end{equation}
and $\tilde{D}^2:=\max_{0\leq i \leq\min\{t,C\}} D_i^2$.
\end{theorem}
Note that $C$ is the smallest iteration index for which one can retain $1/t^2$ rates for the $\mathbb{E}\|A_{\boldsymbol{\xi}}\|^2$ part ($\mathcal{B}$). Without any knowledge about $L_g$, $\mu$ and $\mathbb{E}\|A_{\boldsymbol{\xi}}\|^2$, one can set a parameter $\Omega$ and take $\theta_t = L_g\alpha_t+ \frac{\mu}{2\alpha_t}+\frac{\mathbb{E}\|A_{\boldsymbol{\xi}}\|^2}{\Omega\zeta} - \mu$ in the algorithm. In our experiments, we observe that one can take $\Omega$ fairly large (of $O(\mathbb{E}\|A_{\boldsymbol{\xi}}\|^2)$), meaning that $C$ can be very small (O(1)), and $\mathcal{B}$ is $O(\frac{1}{t^2})$ for \emph{all} $t$. In this sense, strongly convex \textsf{ANSGD} is almost parameter-free. Without the $O(1/t)$ rate of $D_{\mathcal{U}}$, all terms in our bound are optimal. This is why our rate is called ``nearly'' optimal. In practice, $D_{\mathcal{U}}$ is usually small, and it will be dominated by the last term $\frac{\sigma^2}{\mu(t+1)}$.

\subsection{Batch-to-Online Conversion}
The performance of an online learning (online convex minimization) algorithm is typically measured by \emph{regret}, which can be expressed as
\begin{equation}
R(t):= \sum_{i=0}^{t-1} \left[\Phi(\mathbf{x}_{i}, \boldsymbol{\xi}_{i+1}) - \Phi(\mathbf{x}_{t}^*, \boldsymbol{\xi}_{i+1})\right],
\end{equation}
where $\mathbf{x}_t^*:=\arg\min_{\mathbf{x}} \sum_{i=0}^{t-1}\left[\Phi(\mathbf{x}, \boldsymbol{\xi}_{i+1}) \right]$. In the learning theory literature, many approaches are proposed which use online learning algorithms for batch learning (stochastic optimization), called ``online-to-batch'' (O-to-B) conversions. For convex functions, many of these approaches employ an ``averaged'' solution as the final solution.

On the contrary, we show that stochastic optimization algorithms can also be used \emph{directly} for online learning. This ``batch-to-online'' (B-to-O) conversion is almost free of any additional effort: under i.i.d. assumptions of data, one can use any stochastic optimization algorithm for online learning.
%From this perspective, online learning seems to be ``easier'' than stochastic programming.
\begin{proposition}\label{prop:bto}
For any $t\geq 0$, $\mathbb{E}_{\boldsymbol{\xi}_{[t]}} R(t) \leq$
\begin{equation}\label{eq:bto}
\sum_{i=0}^{t-1} \mathbb{E}_{\boldsymbol{\xi}_{[i]}}\left[\Phi(\mathbf{x}_i) - \Phi(\mathbf{x}^*) \right] + \mathbb{E}_{\boldsymbol{\xi}_{[t]}}\sum_{i=0}^{t-1}\left[\Phi(\mathbf{x}_t^*)-\Phi(\mathbf{x}_t^*,\boldsymbol{\xi}_{i+1}) \right]
\end{equation}
where $\mathbf{x}^*:= \arg\min_{\mathbf{x}}\Phi(\mathbf{x})$ and $\mathbf{x}_t^*:=\arg\min_{\mathbf{x}} \sum_{i=0}^{t-1}\left[\Phi(\mathbf{x}, \boldsymbol{\xi}_{i+1}) \right]$.
\end{proposition}
When $\Phi()$ is convex, the second term in (\ref{eq:bto}) can be bounded by applying standard results in uniform convergence (e.g. \cite{boucheron05tcassra}): $\sum_{i=1}^{t-1}\Phi(\mathbf{x}_t^*)-\Phi(\mathbf{x}_t^*,\boldsymbol{\xi}_{i+1}) = O(\sqrt{t})$.
Together with summing up the RHS of (\ref{eq:thm1}), we can obtain an $O(\sqrt{t})$ regret bound.
When $\Phi()$ is strongly convex, the second term in (\ref{eq:bto}) can be bounded using \cite{sss09sco}:
$\sum_{i=1}^{t-1}\Phi(\mathbf{x}_t^*)-\Phi(\mathbf{x}_t^*,\boldsymbol{\xi}_{i+1}) =O(\ln t)$. Together with summing up the RHS of (\ref{eq:thm2}), an $O(\ln t)$ regret bound is achieved. The $O(\sqrt{t})$ and $O(\ln t)$ regret bounds are known

Using our proposed \textsf{ANSGD} for online learning by B-to-O achieves the same (optimal) regret bounds as state-of-the-art algorithms designated for online learning. However, using O-to-B, one can only retain an $O(\ln t/t)$ rate of convergence for stochastic strongly convex optimization. From this perspective, O-to-B is inferior to B-to-O. The sub-optimality of O-to-B is also discussed in \cite{hazan11brmb}.

\section{Examples}\label{sec:examples}
In this section, two nonsmooth functions are given as examples. We will show how these functions can be stochastically approximated, and how to calculate parameters used in our algorithm.
\subsection{Hinge Loss SVM Classification}\label{ssec:hinge}
Hinge loss is a convex surrogate of the $0-1$ loss. Denote a sample-label pair as $\boldsymbol{\xi}:=\{\mathbf{s},l\}\sim P$, where $\mathbf{s}\in\mathbb{R}^D$ and $l\in\mathbb{R}$. Hinge loss can be expressed as $f_{\text{hinge}}(\mathbf{x}):=\max\{0,1-l\mathbf{s}^T\mathbf{x}\}$. It has been widely used for SVM classifiers where the objective is $\min\Phi(\mathbf{x})=\min \mathbb{E}_{\boldsymbol{\xi}}f_{\text{hinge}}(\mathbf{x}) + \frac{\lambda}{2} \|\mathbf{x}\|^2$. Note that the regularization term $g(\mathbf{x})=\frac{\lambda}{2} \|\mathbf{x}\|^2$ is $\lambda$-strongly convex, hence according to Thm.\ref{thm:result_strongly_convex}, \textsf{ANSGD} enjoys $O(1/(\lambda t))$ rates. Taking $\omega(\mathbf{u})=\frac{1}{2}\|\mathbf{u}\|^2$ in (\ref{eq:def_hat_f}), it is easy to check that the smooth stochastic approximation of hinge loss is
\begin{equation}
\hat f_{\text{hinge}}(\mathbf{x},\boldsymbol{\xi}_t,\gamma_t) = \max_{0\leq u\leq 1} \left\{ u\left(1-l_t\mathbf{s}_t^T\mathbf{x}\right) - \gamma_t\frac{u^2}{2} \right\}.
\end{equation}
This maximization is simple enough such that we can obtain an equivalent smooth representation:
\begin{equation}
\hat f_{\text{hinge}}(\mathbf{x},\boldsymbol{\xi}_t,\gamma_t) =
\begin{cases}
0 & \text{if } l_t\mathbf{s}_t^T\mathbf{x} \geq 1,\\
\frac{(1-l_t\mathbf{s}_t^T\mathbf{x})^2}{2\gamma_t} & \text{if } 1-\gamma_t\leq l_t\mathbf{s}_t^T\mathbf{x} < 1,\\
1-l_t\mathbf{s}_t^T\mathbf{x}-\frac{\gamma_t}{2} & \text{if } l_t\mathbf{s}_t^T\mathbf{x} < 1-\gamma_t.
\end{cases}
\end{equation}
Several examples of $\hat f_{\text{hinge}}$ with varying $\gamma_t$ are plotted in Fig.\ref{fig:hinge_and_abs}(left) in comparing with the hinge loss.
\begin{figure}[h!]
\begin{center}
\scalebox{0.65}{\includegraphics*[61pt,274pt][545pt,503pt]{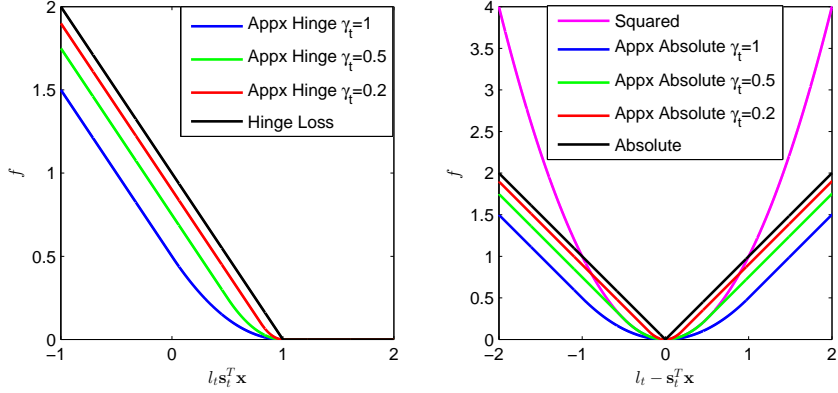}}
\caption{Left: Hinge loss and its smooth approximations. Right: Absolute loss and its smooth approximations.}
\label{fig:hinge_and_abs}
\end{center}
\end{figure}

Here $\mathbf{u}$ is a scalar, hence it is straightforward to calculate $\frac{\mathbb{E}\|A_{\boldsymbol{\xi}}\|^2}{\zeta}$, which will be used to generate sequences $\theta_t$. In binary classification, suppose $l\in\{1,-1\}$. Using definition (\ref{eq:l12norm}), one only needs to calculate $\mathbb{E}(\max_{\|\mathbf{x}\|=1} \mathbf{s}_t^T\mathbf{x})^2$. Practically one can take a small subset of $k$ random samples $\mathbf{s}_i$ (e.g. $k=100$), and calculate the sample average of the squared norms $\frac{1}{k}\sum_{i=1}^k \|\mathbf{s}_i\|^2$. This yields $\frac{1}{k}\sum_{i=1}^k (\max_{\|\mathbf{x}\|=1} \mathbf{s}_i^T\mathbf{x} )^2$, an estimate of $\mathbb{E}\|A_{\boldsymbol{\xi}}\|^2$.

\subsection{Absolute Loss Robust Regression}
Absolute loss is an alternative to the popular squared loss for robust regressions \cite{hastie09esl}. Using same notations as Sec.\ref{ssec:hinge} it can be expressed as $f_{abs}(\mathbf{x}):=|l-\mathbf{s}^T\mathbf{x}|$. Taking $\omega(\mathbf{u})=\frac{1}{2}\|\mathbf{u}\|^2$ in (\ref{eq:def_hat_f}), its smooth stochastic approximation can be expressed as
\begin{equation}
\hat f_{\text{abs}}(\mathbf{x},\boldsymbol{\xi}_t,\gamma_t) = \max_{-1\leq u\leq 1} \left\{ u(l_t-\mathbf{s}_t^T\mathbf{x}) - \gamma_t\frac{u^2}{2}\right\}.
\end{equation}
Solving this maximization wrt $u$ we obtain an equivalent form:
\begin{equation}
\hat f_{\text{abs}}(\mathbf{x},\boldsymbol{\xi}_t,\gamma_t) =
\begin{cases}
l_t-\mathbf{s}_t^T\mathbf{x} - \frac{\gamma_t}{2} & \text{if } l_t-\mathbf{s}_t^T\mathbf{x} \geq \gamma_t,\\
\frac{(l_t-\mathbf{s}_t^T\mathbf{x})^2}{2\gamma_t} & \text{if } -\gamma_t\leq l_t-\mathbf{s}_t^T\mathbf{x}< \gamma_t,\\
-(l_t-\mathbf{s}_t^T\mathbf{x})- \frac{\gamma_t}{2}  & \text{if } l_t-\mathbf{s}_t^T\mathbf{x} < -\gamma_t.
\end{cases}
\end{equation}
This approximation looks similar to the well-studied Huber loss \cite{huber64relp}, though they are different. Actually they share the same form only when $\gamma_t=0.5$ (green curve in Fig.\ref{fig:hinge_and_abs} Right).

The parameter $\mathbb{E}\|A_{\boldsymbol{\xi}}\|^2$ can be estimated in a similar way as discussed in Sec.\ref{ssec:hinge}.

%\section{Related Work}
%\cite{zhang10rrmnagm}, \cite{lin10assgmco}, \cite{duchi11rsso} \cite{yu10qnanco}

\section{Experimental Results}\label{sec:exp}
In this section, five publicly available datasets from various application domains will be used to evaluate the efficiency of \textsf{ANSGD}.  Datasets ``svmguide1'', ``real-sim'', ``rcv1'' and ``alpha'' are for binary classifications, and ``abalone'' is for robust regressions.\footnote{Dataset ``alpha'' is obtained from \url{ftp://largescale.ml.tu-berlin.de/largescale/}, and the other four datasets can be accessed via \url{http://www.csie.ntu.edu.tw/~cjlin/libsvmtools}. Dataset ``rcv1'' comes with $20,242$ training samples and $677,399$ testing samples. For ``svmguide1'' and ``real-sim'', we randomly take $60\%$ of the samples for training and $40\%$ for testing. For ``alpha'' and ``abalone'', $80\%$ are used for training, and the rest $20\%$ are used for testing.}

Following our examples in Sec.\ref{sec:examples}, we will evaluate our algorithm using approximated hinge loss for classifications, and approximated absolute loss for regressions. Exact hinge and absolute losses will be used for subgradient descent algorithms that we will compare with, as described in the following section. All losses are squared-$l2$-norm-regularized. The regularization parameter $\lambda$ is shown on each figure. When assuming strong-convexity, we take $\mu=\lambda$.
\subsection{Algorithms for Comparison and Parameters}\label{subsec:alg_cmp}
We compare \textsf{ANSGD} with three state-of-the-art algorithms. Each algorithm has a data-dependent tuning parameter, denoted by $\Omega$ (although they have different physical meanings). The best values of $\Omega$ are found based on a tuning subset of samples. Note that when assuming strong-convexity, our \textsf{ANSGD} is almost parameter-free. As discussed after Thm.\ref{thm:result_strongly_convex}, our experiments indicate that the optimal $\Omega$ is taken such that $\frac{\mathbb{E}\|A_{\boldsymbol{\xi}}\|^2}{\Omega\zeta} \approx 1$, meaning that one can simply take $\theta_t=L_g\alpha_t+\frac{\mu}{2\alpha_t}+1-\mu$.

\textsf{SGD}. The classic stochastic approximation \cite{robbins51sam} is adopted: $\mathbf{x}_{t+1}\leftarrow \mathbf{x}_t -\eta_t f^{\prime}(\mathbf{x}_t)$, where $f^{\prime}(\mathbf{x}_t)$ is the subgradient. When only assuming convexity ($\mu=0$), we use stepsize $\eta_t=\frac{\Omega}{\sqrt{t}}$. When assuming strong-convexity, we follow the stepsize used in SGD2 \cite{bottousgd2}: $\eta_t=\frac{1}{\mu(t+\Omega)}$.

\textsf{Averaged SGD}. This is algorithmically the same as \textsf{SGD}, except that the averaged result $\bar{\mathbf{x}}:=\frac{1}{t}\sum_{i=1}^t\mathbf{x}_i$ is used for testing. We follow the stepsizes suggested by the recent work on the non-asymptotic analysis of SGD \cite{bach11naasaaml,xu11tooplslasgd}, where it is argued that Polyak's averaging combining with proper stepsizes yield optimal rates. When only assuming convexity, we use stepsizes $\eta_t = \frac{\Omega}{\sqrt{t}}$ \cite{bach11naasaaml}. When assuming strong convexity, the stepsize is taken as $\eta_t = \frac{1}{\Omega (1+\mu t/\Omega)^{3/4}} $\cite{xu11tooplslasgd}.

%Although both versions are originally proposed for smooth minimizations, we still include them for comparison. The reason is that
\textsf{AC-SA}. This approach \cite{lan10omsco,lan11osaascsco1} is interesting to compare because like \textsf{ANSGD}, it is another way of obtaining a stochastic algorithm based on Nesterov's optimal method, begging the question of whether it has similar behavior. Theoretically, according to Prop.8 and 9 in \cite{lan11osaascsco1}, the bound for the nonsmooth part is of $O(1/\sqrt{t})$ for $\mu=0$ and $O(1/t)$ for $\mu>0$. In comparison, our nonsmooth part converges in $O(1/t)$ for $\mu=0$ and $O(1/t^2)$ for $\mu>0$. Numerically we observe that directly applying \textsf{AC-SA} to nonsmooth functions results in inferior performances.
%The details of this algorithm are presented in \cite{lan10omsco}, and the strongly convex version is proposed in \cite{lan11osaascsco1}. Both \textsf{AC-SA} and our \textsf{ANSGD} are stochastic algorithms based on Nesterov's optimal methods, hence some readers might wonder that they may performance similarly.

\subsection{Results}
Due to the stochasticity of all the algorithms, for each setting of the experiments, we run the program for $10$ times, and plot the mean and standard deviation of the results using error bars.

In the first set of experiments, we compare \textsf{ANSGD} with two subgradient-based algorithms \textsf{SGD} and \textsf{Averaged SGD}. Classification results are shown in Fig.\ref{fig:svmguide1_vs_SGD}, \ref{fig:realsim_vs_SGD}, \ref{fig:rcv1_vs_SGD} and \ref{fig:alpha_vs_SGD}, and regression results are shown in Fig.\ref{fig:abalone_loss_vs_SGD}. In each figure, the left column is for algorithms without strongly convex assumptions, while in the right column the algorithms assume strong-convexity and take $\mu=\lambda$. For classification results, we plot function values over the testing set in the first row, and plot testing accuracies in the second row.
\begin{figure}[h!]
\begin{center}
\scalebox{0.7}{\includegraphics*[130pt,248pt][484pt,549pt]{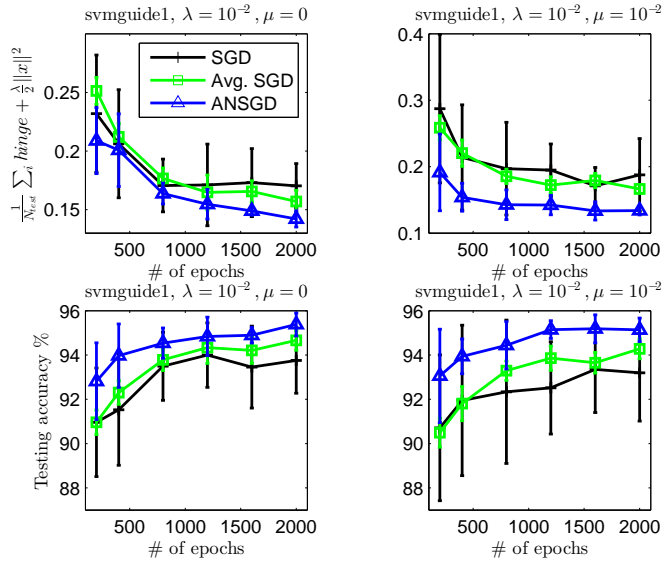}}
\caption{Classification with ``svmguide1''.}
\label{fig:svmguide1_vs_SGD}
\end{center}
\end{figure}
\begin{figure}[h!]
\begin{center}
\scalebox{0.7}{\includegraphics*[130pt,248pt][484pt,549pt]{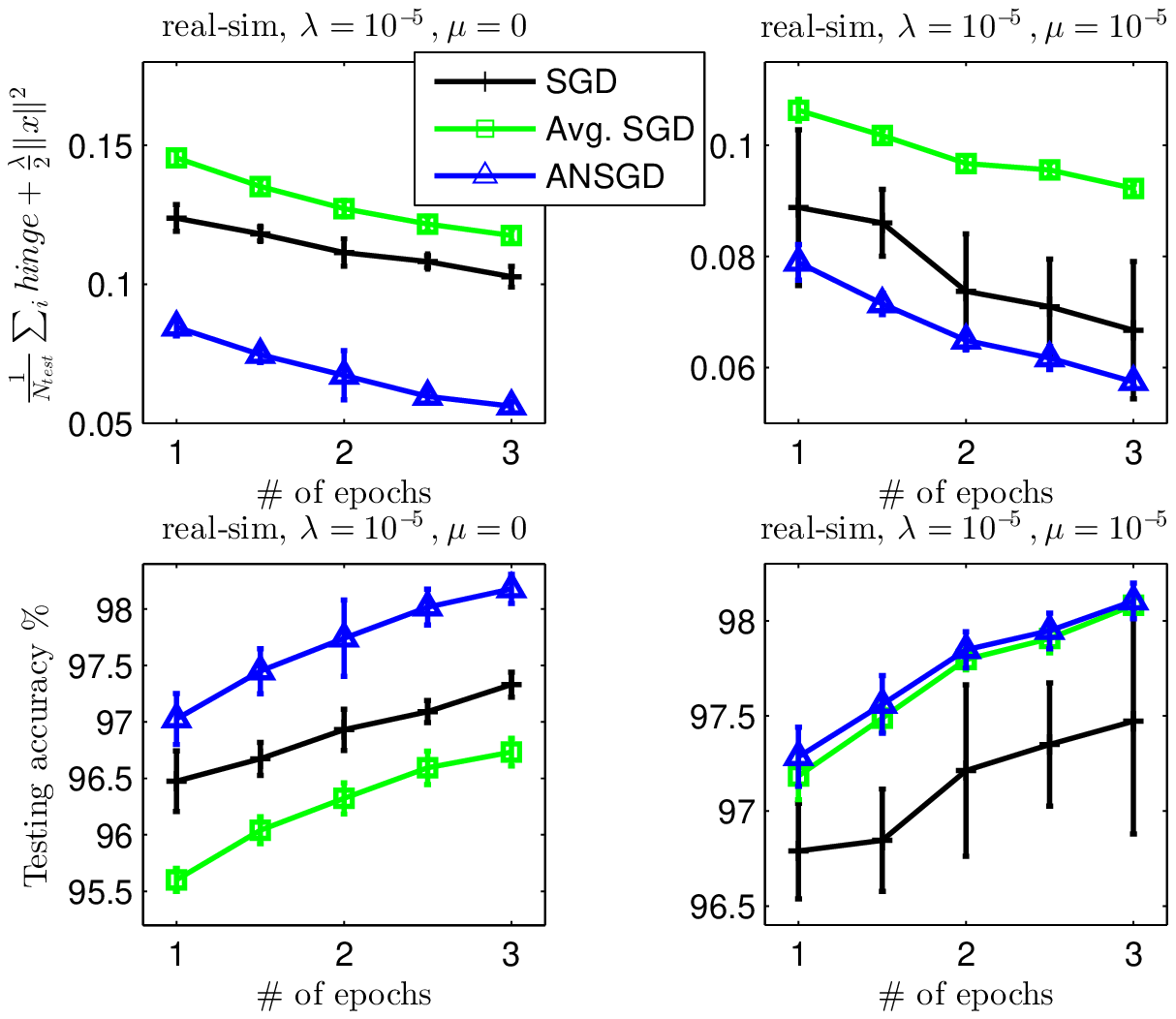}}
\caption{Classification with ``real-sim''.}
\label{fig:realsim_vs_SGD}
\end{center}
\end{figure}
\begin{figure}[h!]
\begin{center}
\scalebox{0.7}{\includegraphics*[130pt,248pt][484pt,549pt]{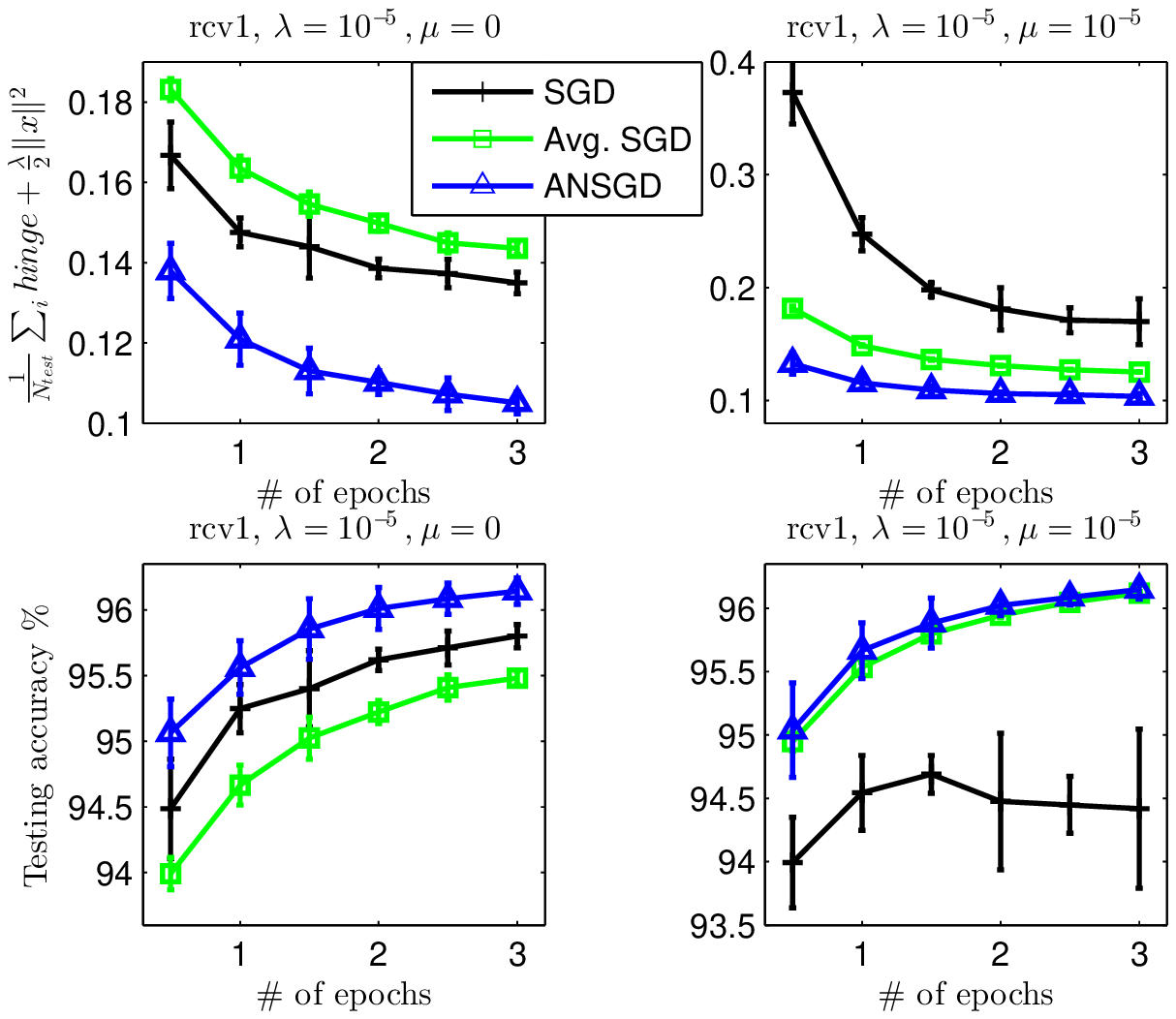}}
\caption{Classification with ``rcv1''.}
\label{fig:rcv1_vs_SGD}
\end{center}
\end{figure}
\begin{figure}[h!]
\begin{center}
\scalebox{0.7}{\includegraphics*[130pt,248pt][484pt,549pt]{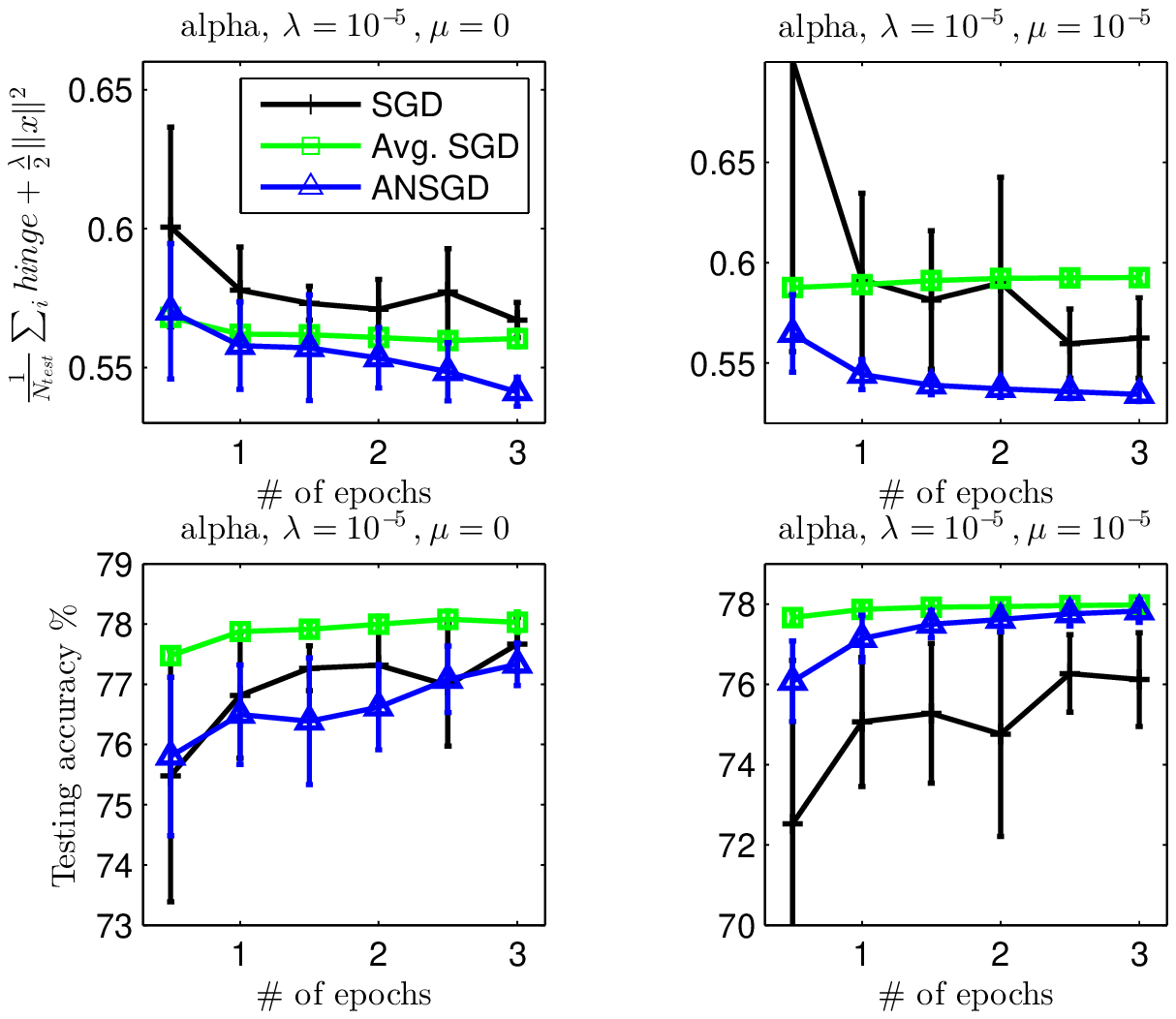}}
\caption{Classification with ``alpha''.}
\label{fig:alpha_vs_SGD}
\end{center}
\end{figure}
\begin{figure}[h!]
\begin{center}
\scalebox{0.7}{\includegraphics*[139pt,314pt][479pt,474pt]{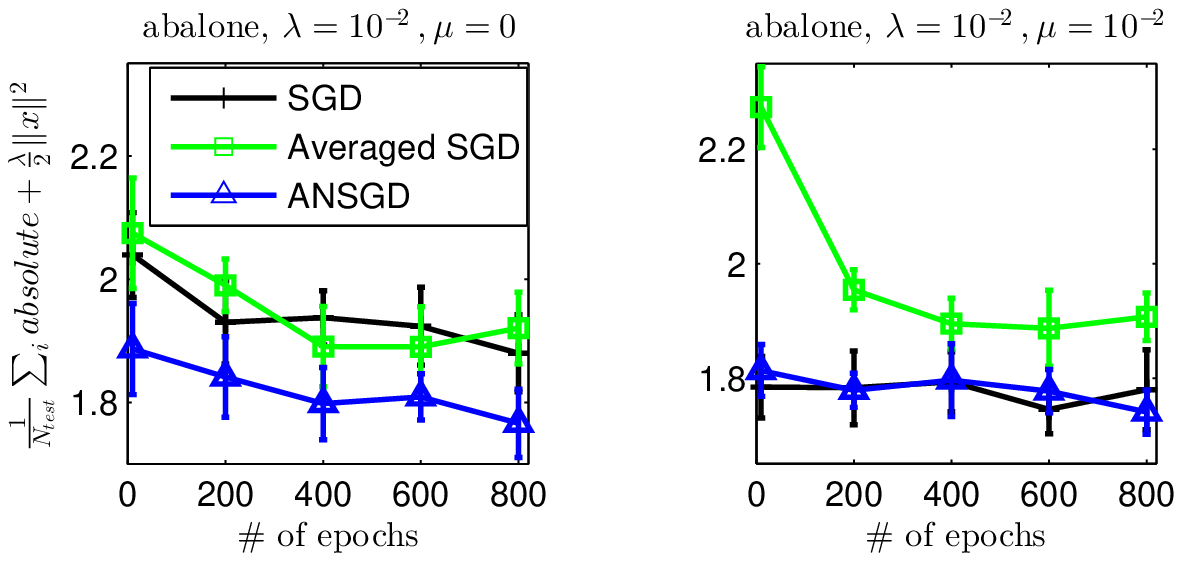}}
\caption{Regression with ``abalone''.}
\label{fig:abalone_loss_vs_SGD}
\end{center}
\end{figure}

It is clear that in all these experiments, \textsf{ANSGD}'s function values converges consistently faster than the other two SGD algorithms. In non-strongly convex experiments, it converges significantly faster than \textsf{SGD} and its averaged version. In strongly convex experiments, it still out performs, and is more robust than strongly convex \textsf{SGD}. \textsf{Averaged SGD} performs well in strongly convex settings, in terms of prediction accuracies, although its errors are still higher than \textsf{ANSGD} in the first three datasets. The only exception is in ``alpha'' (Fig.\ref{fig:alpha_vs_SGD}), where \textsf{Averaged SGD} retains higher function values than \textsf{ANSGD}, but its accuracies are contradictorily higher in early stages. The reason might be that the inexact solution serves as an additional regularization factor, which cannot be predicted by the analysis of convergence rates.

In the second set of experiments, we compare \textsf{ANSGD} with \textsf{AC-SA} and its strongly convex version. Results are in Fig.\ref{fig:svmguide1_vs_ACSA}, \ref{fig:realsim_vs_ACSA}, \ref{fig:rcv1_vs_ACSA} and \ref{fig:alpha_vs_ACSA}. In all experiments our \textsf{ANSGD} significantly outperforms \text{AC-SA}, and is much more stable. These experiments confirm the theoretically better rates discussed in Sec.\ref{subsec:alg_cmp}.
\begin{figure}[h!]
\begin{center}
\scalebox{0.7}{\includegraphics*[130pt,248pt][484pt,549pt]{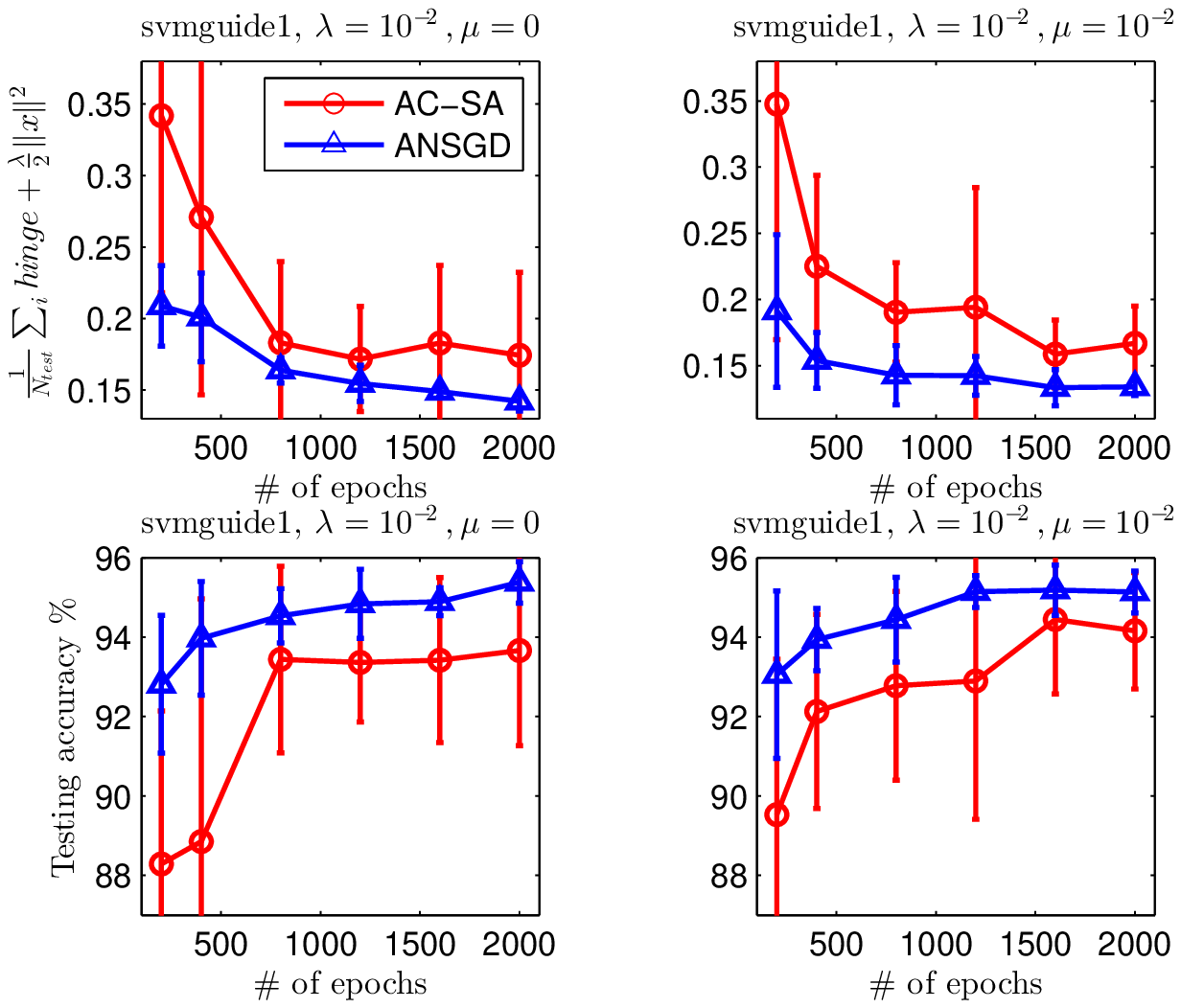}}
\caption{Classification with ``svmguide1''.}
\label{fig:svmguide1_vs_ACSA}
\end{center}
\end{figure}
\begin{figure}[h!]
\begin{center}
\scalebox{0.7}{\includegraphics*[130pt,248pt][484pt,549pt]{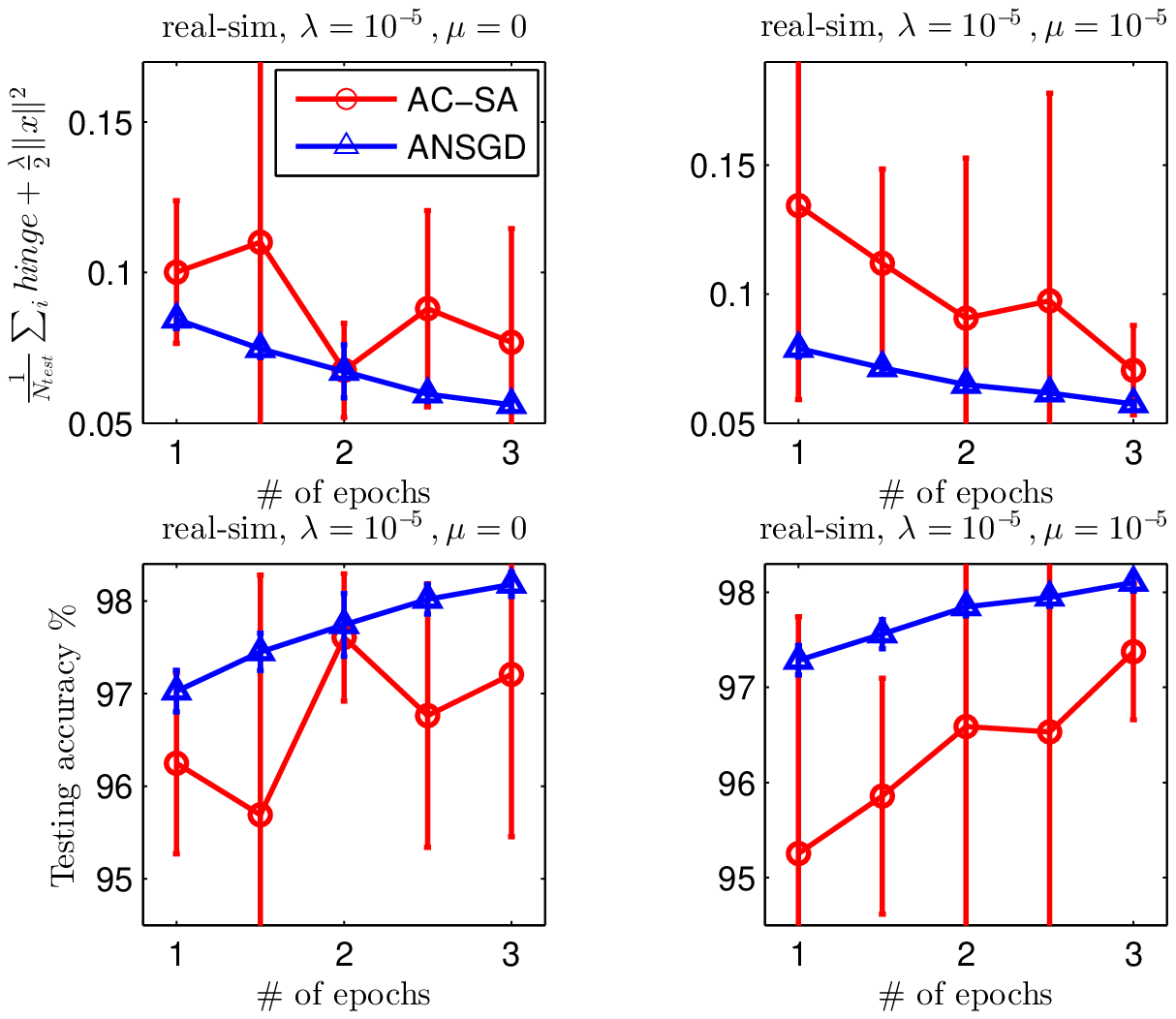}}
\caption{Classification with ``real-sim''.}
\label{fig:realsim_vs_ACSA}
\end{center}
\end{figure}
\begin{figure}[h!]
\begin{center}
\scalebox{0.7}{\includegraphics*[130pt,248pt][484pt,549pt]{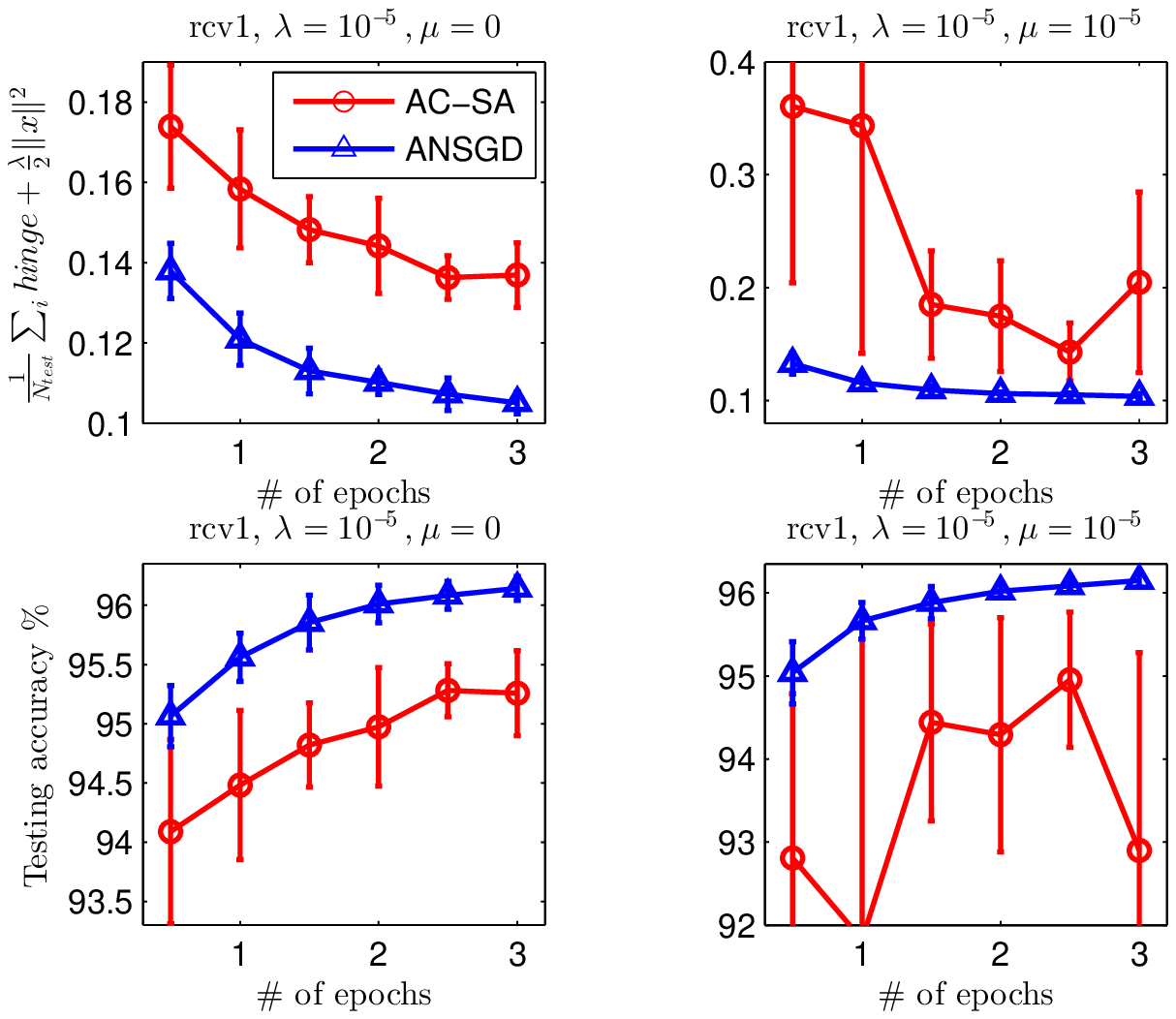}}
\caption{Classification with ``rcv1''.}
\label{fig:rcv1_vs_ACSA}
\end{center}
\end{figure}
\begin{figure}[h!]
\begin{center}
\scalebox{0.7}{\includegraphics*[130pt,248pt][484pt,549pt]{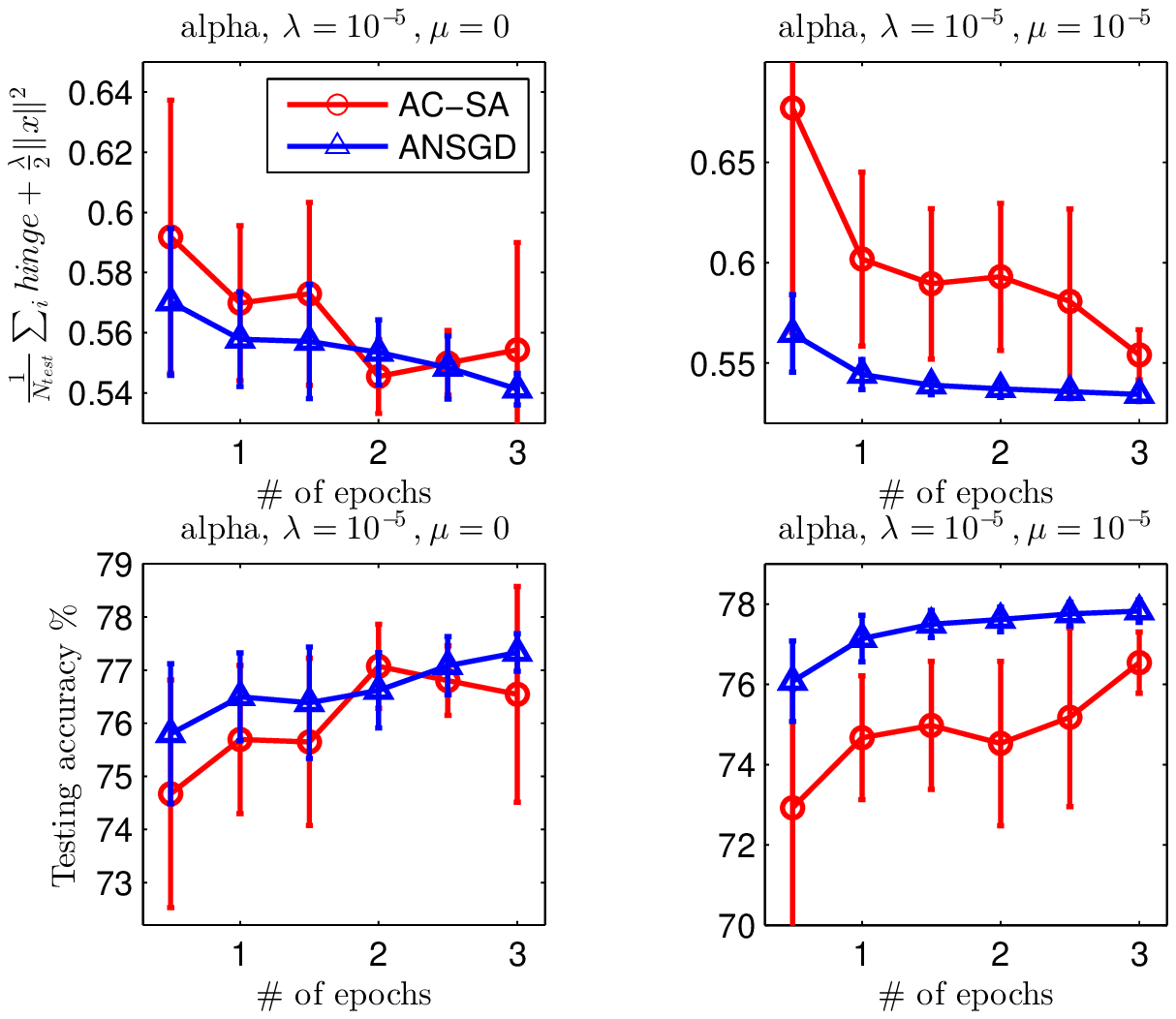}}
\caption{Classification with ``alpha''.}
\label{fig:alpha_vs_ACSA}
\end{center}
\end{figure}

\section{Conclusions and Future Work}
We introduce a different composite setting for nonsmooth functions. Under this setting we propose a stochastic smoothing method and a novel stochastic algorithm \textsf{ANSGD}. Convergence analysis show that it achieves (nearly) optimal rates under both convex and strongly convex assumptions. We also propose a ``Batch-to-Online'' conversion for online learning, and show that optimal regrets can be obtained.

We will extend our method to constrained minimizations, as well as cases when the approximated function $\hat{f}()$ is not easily obtained by maximizing $\mathbf{u}$. Nesterov's excessive gap technique has the ``true'' optimal $1/t^2$ bound, and we will investigate the possibility of integrating it in our algorithm. Exploiting links with statistical learning theories may also be promising.

%---------------------------------- APPENDIX ---------------------------------
\onecolumn
\appendix
\section{Proof of Lemma \ref{lm:agd_phi}}
\begin{proof}
\begin{equation}
\begin{split}
&\Phi(\mathbf{x}_t) - \Phi(\mathbf{x}) \\
&= \left[f(\mathbf{x}_t) - f(\mathbf{x})\right] + \left[g(\mathbf{x}_t) - g(\mathbf{x})\right]\\
&=\mathbb{E}_{\boldsymbol{\xi}}\left[f(\mathbf{x}_t,\boldsymbol{\xi}) \right] + \mathbb{E}_{\boldsymbol{\xi}}\left[-f(\mathbf{x},\boldsymbol{\xi}) + g(\mathbf{x}_t,\boldsymbol{\xi}) - g(\mathbf{x},\boldsymbol{\xi}) \right]\\
&= \mathbb{E}_{\boldsymbol{\xi}} \max_{\mathbf{u}\in\mathcal{U}} \bigg\{ \big[\langle A_{\boldsymbol{\xi}}\mathbf{x}_t,\mathbf{u} \rangle -Q(\mathbf{u})-\gamma_t\omega(\mathbf{u}) \big] + \gamma_t\omega(\mathbf{u}) \bigg\} + \mathbb{E}_{\boldsymbol{\xi}}\left[-f(\mathbf{x},\boldsymbol{\xi}) + g(\mathbf{x}_t,\boldsymbol{\xi}) - g(\mathbf{x},\boldsymbol{\xi}) \right]\\
&\leq \mathbb{E}_{\boldsymbol{\xi}} \max_{\mathbf{u}\in\mathcal{U}} \big[\langle A_{\boldsymbol{\xi}}\mathbf{x}_t,\mathbf{u} \rangle -Q(\mathbf{u})-\gamma_t\omega(\mathbf{u}) \big] + \max_{\mathbf{u}\in\mathcal{U}}\big[ \gamma_t\omega(\mathbf{u}) \big] + \mathbb{E}_{\boldsymbol{\xi}}\left[-f(\mathbf{x},\boldsymbol{\xi}) + g(\mathbf{x}_t,\boldsymbol{\xi}) - g(\mathbf{x},\boldsymbol{\xi}) \right]\\
&= \mathbb{E}_{\boldsymbol{\xi}}\left[\hat{f}(\mathbf{x}_t,\boldsymbol{\xi},\gamma_t)\right] + \gamma_t D_{\mathcal{U}}  + \mathbb{E}_{\boldsymbol{\xi}}\left[-f(\mathbf{x},\boldsymbol{\xi}) + g(\mathbf{x}_t,\boldsymbol{\xi}) - g(\mathbf{x},\boldsymbol{\xi}) \right]\\
&\leq \mathbb{E}_{\boldsymbol{\xi}}\left[\hat{f}(\mathbf{x}_t,\boldsymbol{\xi},\gamma_t) -\hat{f}(\mathbf{x},\boldsymbol{\xi},\gamma_t)\right] + \mathbb{E}_{\boldsymbol{\xi}}\left[g(\mathbf{x}_t,\boldsymbol{\xi}) - g(\mathbf{x},\boldsymbol{\xi}) \right] + \gamma_t D_{\mathcal{U}}.
\end{split}
\end{equation}
The last inequality is due to the non-negativity of $\omega()$ and definitions of $f$ (\ref{eq:def_f}) and $\hat{f}$ (\ref{eq:def_hat_f}).
\end{proof}

\section{Proof of Lemma \ref{lm:agd_lm1}}
Before proceeding to the proof of this lemma, we present two auxiliary results. For clarity, in the following lemmas and proofs we use the following notations to denote the smoothly approximated composite function and its expectation:
\begin{equation} F_t(\mathbf{x},\gamma_t):=\hat{f}_t(\mathbf{x})+g_t(\mathbf{x}) = \hat{f}(\mathbf{x},\boldsymbol{\xi}_t,\gamma_t)+g(\mathbf{x},\boldsymbol{\xi}_t)
\end{equation}
and
\begin{equation}
F(\mathbf{x},\gamma_t):=\mathbb{E}_{\boldsymbol{\xi}_t} F_t(\mathbf{x},\gamma_t).
\end{equation}
The first lemma is on the smoothly approximated function and the smoothness parameter $\gamma_t$.
\begin{lemma}\label{lm:gamma}
If $\gamma_t$ is monotonically decreasing with $t$, for any $\mathbf{x}$ and $t\geq 0$,
\begin{equation}
F(\mathbf{x},\gamma_{t}) \leq F(\mathbf{x},\gamma_{t+1}) \leq F(\mathbf{x},\gamma_{t}) + (\gamma_t-\gamma_{t+1})D_{\mathcal{U}},
\end{equation}
where $D_{\mathcal{U}}:=\max_{\mathbf{u}\in\mathcal{U}}\omega(\mathbf{u})$.
\end{lemma}
\begin{proof}
The left inequality is obvious, since $\gamma_t\geq \gamma_{t+1}$ and $\omega(\mathbf{u})$ is nonnegative. For the right inequality,
\begin{equation}
\begin{split}
F(\mathbf{x},\gamma_{t+1}) - F(\mathbf{x},\gamma_{t}) &= \mathbb{E}_{\boldsymbol{\xi}}\hat{f}(\mathbf{x},\boldsymbol{\xi},\gamma_{t+1}) - \mathbb{E}_{\boldsymbol{\xi}}\hat{f}(\mathbf{x},\boldsymbol{\xi},\gamma_t)\\
& = \max_{\mathbf{u}\in\mathcal{U}}\left[ \langle \mathbb{E}_{\boldsymbol{\xi}}A_{\boldsymbol{\xi}}\mathbf{x}, \mathbf{u}\rangle -Q(\mathbf{u}) -\gamma_{t+1}\omega(\mathbf{u})\right] - \max_{\mathbf{u}\in\mathcal{U}}\left[ \langle \mathbb{E}_{\boldsymbol{\xi}}A_{\boldsymbol{\xi}}\mathbf{x}, \mathbf{u}\rangle -Q(\mathbf{u}) -\gamma_{t}\omega(\mathbf{u})\right]\\
&\leq \max_{\mathbf{u}\in\mathcal{U}}\bigg\{ \big[\langle \mathbb{E}_{\boldsymbol{\xi}}A_{\boldsymbol{\xi}}\mathbf{x}, \mathbf{u}\rangle -Q(\mathbf{u}) -\gamma_{t+1}\omega(\mathbf{u})\big] -  \big[\langle \mathbb{E}_{\boldsymbol{\xi}}A_{\boldsymbol{\xi}}\mathbf{x}, \mathbf{u}\rangle -Q(\mathbf{u}) -\gamma_{t}\omega(\mathbf{u})\big] \bigg\}\\
&= \max_{\mathbf{u}\in\mathcal{U}} \left[ (\gamma_t-\gamma_{t+1}) \omega(\mathbf{u})\right].
\end{split}
\end{equation}
\end{proof}

The second lemma is about proximal methods using Bregman divergence as prox-functions, which is a direct result of optimality conditions. It appeared in \cite{lan11osaascsco1}(Lemma 2), and is an extension of the ``$3$-point identity'' \cite{chen93threepoint}(Lemma 3.1).
\begin{lemma}\cite{lan11osaascsco1}\label{lm:3point_ext}
Let $l(\mathbf{x})$ be a convex function. Let scalars $s_1,s_2\geq 0$. For any vectors $\mathbf{u}$ and $\mathbf{v}$, denote their Bregman divergence as $D(\mathbf{u},\mathbf{v})$. If $\forall \mathbf{x},\mathbf{u},\mathbf{v}$
\begin{equation}
\mathbf{x}^* = \arg\min_{\mathbf{x}} l(\mathbf{x}) + s_1 D(\mathbf{u},\mathbf{x}) + s_2 D(\mathbf{v},\mathbf{x}),
\end{equation}
then
\begin{equation}
l(\mathbf{x}) + s_1 D(\mathbf{u},\mathbf{x}) + s_2 D(\mathbf{v},\mathbf{x}) \geq l(\mathbf{x}^*) + s_1 D(\mathbf{u},\mathbf{x}^*) + s_2 D(\mathbf{v},\mathbf{x}^*) + (s_1+s_2) D(\mathbf{x}^*,\mathbf{x}).
\end{equation}
\end{lemma}

We are now ready to prove Lemma \ref{lm:agd_lm1}.
\begin{proof}[Proof of Lemma \ref{lm:agd_lm1}]
Due to Lemma \ref{lm:exp_lips} and Lipschitz-smoothness of $g(\mathbf{x})$, $F(\mathbf{x},\gamma_{t+1})$ has a Lipschitz smooth constant $L_{F_{t+1}}:= \frac{\mathbb{E}_{\boldsymbol{\xi}}\|A_{\boldsymbol{\xi}}\|^2}{\gamma_{t+1}\zeta}+L_g$. It follows that
\begin{equation}
\begin{split}
&F(\mathbf{x}_{t+1},\gamma_{t+1})\\
&\leq F(\mathbf{y}_t,\gamma_{t+1}) + \langle \nabla F(\mathbf{y}_t,\gamma_{t+1}), \mathbf{x}_{t+1}- \mathbf{y}_t \rangle + \frac{L_{F_{t+1}}}{2}\|\mathbf{x}_{t+1}- \mathbf{y}_t\|^2\\
&=(1-\alpha_t)F(\mathbf{y}_t,\gamma_{t+1}) + \alpha_t F(\mathbf{y}_t,\gamma_{t+1}) + \langle \nabla F(\mathbf{y}_t,\gamma_{t+1}), \mathbf{x}_{t+1}- \mathbf{y}_t \rangle + \frac{L_{F_{t+1}}}{2}\|\mathbf{x}_{t+1}- \mathbf{y}_t\|^2\\
&=(1-\alpha_t)F(\mathbf{y}_t,\gamma_{t+1}) + \langle \nabla F(\mathbf{y}_t,\gamma_{t+1}), (1-\alpha_t)(\mathbf{x}_{t}-\mathbf{y}_{t})\rangle +\\
&\ \ \ \ \alpha_t F(\mathbf{y}_t,\gamma_{t+1}) + \langle \nabla F(\mathbf{y}_t,\gamma_{t+1}), \mathbf{x}_{t+1}- \mathbf{y}_t -(1-\alpha_t)(\mathbf{x}_{t}-\mathbf{y}_{t}) \rangle + \frac{L_{F_{t+1}}}{2}\|\mathbf{x}_{t+1}- \mathbf{y}_t\|^2\\
&\leq (1-\alpha_t)F(\mathbf{x}_t,\gamma_{t+1}) + \alpha_t F(\mathbf{y}_t,\gamma_{t+1}) + \langle \nabla F(\mathbf{y}_t,\gamma_{t+1}), \mathbf{x}_{t+1}- \mathbf{y}_t -(1-\alpha_t)(\mathbf{x}_{t}-\mathbf{y}_{t}) \rangle +\\ &\ \ \ \ \frac{L_{F_{t+1}}}{2}\|\mathbf{x}_{t+1}- \mathbf{y}_t\|^2,
\end{split}
\end{equation}
where the last inequality is due to the convexity of $F()$. Subtracting $F(\mathbf{x},\gamma_{t+1})$ from both sides of the above inequality we have:
\begin{equation}\label{eq:lm_main_auxx}
\begin{split}
F(\mathbf{x}_{t+1},\gamma_{t+1}) &- F(\mathbf{x},\gamma_{t+1}) \leq (1-\alpha_t)F(\mathbf{x}_t,\gamma_{t+1}) - F(\mathbf{x},\gamma_{t+1})\\
&+ \alpha_t F(\mathbf{y}_t,\gamma_{t+1}) + \langle \nabla F(\mathbf{y}_t,\gamma_{t+1}), \mathbf{x}_{t+1}- \mathbf{y}_t -(1-\alpha_t)(\mathbf{x}_{t}-\mathbf{y}_{t}) \rangle + \frac{L_{F_{t+1}}}{2}\|\mathbf{x}_{t+1}- \mathbf{y}_t\|^2\\
&\leq (1-\alpha_t)\big[F(\mathbf{x}_t,\gamma_{t})+(\gamma_t-\gamma_{t+1})D_{\mathcal{U}}\big] - F(\mathbf{x},\gamma_{t+1})\\
&+ \alpha_t F(\mathbf{y}_t,\gamma_{t+1}) + \langle \nabla F(\mathbf{y}_t,\gamma_{t+1}), \mathbf{x}_{t+1}- \mathbf{y}_t -(1-\alpha_t)(\mathbf{x}_{t}-\mathbf{y}_{t}) \rangle + \frac{L_{F_{t+1}}}{2}\|\mathbf{x}_{t+1}- \mathbf{y}_t\|^2\\
&\leq (1-\alpha_t)\big[ F(\mathbf{x}_t,\gamma_{t}) - F(\mathbf{x},\gamma_{t}) \big] -\alpha_t F(\mathbf{x},\gamma_{t+1}) +(1-\alpha_t)(\gamma_t-\gamma_{t+1})D_{\mathcal{U}} \\
&+ \alpha_t F(\mathbf{y}_t,\gamma_{t+1}) + \langle \nabla F(\mathbf{y}_t,\gamma_{t+1}), \mathbf{x}_{t+1}- \mathbf{y}_t -(1-\alpha_t)(\mathbf{x}_{t}-\mathbf{y}_{t}) \rangle + \frac{L_{F_{t+1}}}{2}\|\mathbf{x}_{t+1}- \mathbf{y}_t\|^2,
\end{split}
\end{equation}
where the last two inequalities are due to Lemma \ref{lm:gamma}.

Denoting $\Delta_t := F(\mathbf{x}_t,\gamma_t) - F(\mathbf{x},\gamma_t)$ and $\sigma_t(\mathbf{x}) := \nabla F_t(\mathbf{x},\gamma_t)-\nabla F(\mathbf{x},\gamma_t)$ we can rewrite (\ref{eq:lm_main_auxx}) as:
\begin{equation}\label{eq:main_lm_Delta}
\begin{split}
&\Delta_{t+1} - (1-\alpha_t)\Delta_t - (1-\alpha_t)(\gamma_t-\gamma_{t+1})D_{\mathcal{U}}\\
&\leq \alpha_t F(\mathbf{y}_t,\gamma_{t+1}) - \alpha_t F(\mathbf{x},\gamma_{t+1}) + \langle \nabla F(\mathbf{y}_t,\gamma_{t+1}), \mathbf{x}_{t+1}- \mathbf{y}_t -(1-\alpha_t)(\mathbf{x}_{t}-\mathbf{y}_{t}) \rangle + \frac{L_{F_{t+1}}}{2}\|\mathbf{x}_{t+1}- \mathbf{y}_t\|^2\\
&\overset{(\ref{eq:g_mu_str})}{\leq} \alpha_t F(\mathbf{y}_t,\gamma_{t+1}) - \alpha_t\left[ F(\mathbf{y}_t,\gamma_{t+1})+\langle \nabla F(\mathbf{y}_t,\gamma_{t+1}), \mathbf{x}-\mathbf{y}_t \rangle + \frac{\mu}{2}\|\mathbf{x}-\mathbf{y}_t\|^2\right] +\\
&\ \ \ \ \langle \nabla F(\mathbf{y}_t,\gamma_{t+1}), \mathbf{x}_{t+1}- \mathbf{y}_t -(1-\alpha_t)(\mathbf{x}_{t}-\mathbf{y}_{t}) \rangle + \frac{L_{F_{t+1}}}{2}\|\mathbf{x}_{t+1}- \mathbf{y}_t\|^2\\
&= -\alpha_t\left[ \langle \nabla F_{t+1}(\mathbf{y}_t,\gamma_{t+1}) - \sigma_{t+1}(\mathbf{y}_t), \mathbf{x}-\mathbf{y}_t\rangle + \frac{\mu}{2}\|\mathbf{x}-\mathbf{y}_t\|^2\right]+\\
&\ \ \ \ \langle \nabla F(\mathbf{y}_t,\gamma_{t+1}), \mathbf{x}_{t+1}- \mathbf{y}_t -(1-\alpha_t)(\mathbf{x}_{t}-\mathbf{y}_{t}) \rangle + \frac{L_{F_{t+1}}}{2}\|\mathbf{x}_{t+1}- \mathbf{y}_t\|^2\\
&= -\alpha_t\left[ \langle \nabla F_{t+1}(\mathbf{y}_t,\gamma_{t+1}), \mathbf{x}-\mathbf{y}_t\rangle + \frac{\mu}{2}\|\mathbf{x}-\mathbf{y}_t\|^2 + \frac{\theta_t}{2}\|\mathbf{x}-\mathbf{v}_t\|^2 \right] + \frac{\alpha_t\theta_t}{2}\|\mathbf{x}-\mathbf{v}_t\|^2+\\
&\ \ \ \ \langle \nabla F(\mathbf{y}_t,\gamma_{t+1}), \mathbf{x}_{t+1}- \mathbf{y}_t -(1-\alpha_t)(\mathbf{x}_{t}-\mathbf{y}_{t}) \rangle + \frac{L_{F_{t+1}}}{2}\|\mathbf{x}_{t+1}- \mathbf{y}_t\|^2 + \langle \sigma_{t+1}(\mathbf{y}_t),\alpha_t(\mathbf{x}-\mathbf{y}_t) \rangle\\
&\leq -\alpha_t\left[ \langle \nabla F_{t+1}(\mathbf{y}_t,\gamma_{t+1}), \mathbf{v}_{t+1}-\mathbf{y}_t\rangle + \frac{\mu}{2}\|\mathbf{v}_{t+1}-\mathbf{y}_t\|^2 + \frac{\theta_t}{2}\|\mathbf{v}_{t+1}-\mathbf{v}_t\|^2 + \frac{\mu+\theta_t}{2}\|\mathbf{x}-\mathbf{v}_{t+1}\|^2 \right] +\\
&\ \ \ \ \frac{\alpha_t\theta_t}{2}\|\mathbf{x}-\mathbf{v}_t\|^2+ \langle \nabla F(\mathbf{y}_t,\gamma_{t+1}), \mathbf{x}_{t+1}- \mathbf{y}_t -(1-\alpha_t)(\mathbf{x}_{t}-\mathbf{y}_{t}) \rangle + \frac{L_{F_{t+1}}}{2}\|\mathbf{x}_{t+1}- \mathbf{y}_t\|^2 +\\
 &\ \ \ \ \langle \sigma_{t+1}(\mathbf{y}_t),\alpha_t(\mathbf{x}-\mathbf{y}_t) \rangle,
\end{split}
\end{equation}
where the last inequality is due to Lemma \ref{lm:3point_ext} (taking $D(\mathbf{u},\mathbf{v})=\frac{1}{2}\|\mathbf{u}-\mathbf{v}\|^2$) and the definition of $\mathbf{v}_{t+1}$:
\begin{equation}
\mathbf{v}_{t+1} := \arg\min_{\mathbf{x}} \langle \nabla F_{t+1}(\mathbf{y}_t,\gamma_{t+1}), \mathbf{x}-\mathbf{y}_t\rangle + \frac{\mu}{2}\|\mathbf{x}-\mathbf{y}_t\|^2 + \frac{\theta_t}{2}\|\mathbf{x}-\mathbf{v}_t\|^2.
\end{equation}
Minimizing the above directly leads to Line 4 of Alg.\ref{alg:ansgd}:
\begin{equation}\label{eq:main_lm_v_t1}
\mathbf{v}_{t+1} = \frac{\theta_t \mathbf{v}_t + \mu\mathbf{y}_t - \nabla F_{t+1}(\mathbf{y}_t,\gamma_{t+1})}{\mu+\theta_t}.
\end{equation}
Base on this updating rule, it is easy to verify the following inequality:
\begin{equation}\label{eq:main_lm_aux_ineq}
\begin{split}
&-\alpha_t\left[ \frac{\mu}{2}\|\mathbf{v}_{t+1}-\mathbf{y}_t\|^2 + \frac{\theta_t}{2}\|\mathbf{v}_{t+1}-\mathbf{v}_t\|^2 \right] \\
&\leq -\frac{\alpha_t}{2}\left[ \frac{\mu\theta_t}{\mu+\theta_t}\|\mathbf{v}_t-\mathbf{y}_t\|^2 + \frac{1}{\mu+\theta_t}\|\nabla F_{t+1}(\mathbf{y}_t,\gamma_{t+1})\|^2\right]\leq \frac{-\alpha_t}{2\left(\mu+\theta_t\right)}\|\nabla F_{t+1}(\mathbf{y}_t,\gamma_{t+1})\|^2.
\end{split}
\end{equation}
To set $\mathbf{x}_{t+1}$ (Line 3 of Alg.\ref{alg:ansgd}), we follow the classic stochastic gradient descent, such that $\|\mathbf{x}_{t+1}-\mathbf{y}_t\|^2$ can be bounded in terms of $\|\nabla F_{t+1}(\mathbf{y}_t,\gamma_{t+1})\|^2$:
$\mathbf{x}_{t+1} = \mathbf{y}_t - \eta_t \nabla F_{t+1}(\mathbf{y}_t,\gamma_{t+1})$.

Hence
\begin{equation}\label{lm:main_lm_proj}
\|\mathbf{x}_{t+1}-\mathbf{y}_t\|^2 = \eta_t^2\|\nabla F_{t+1}(\mathbf{y}_t,\gamma_{t+1})\|^2,
\end{equation}
and
\begin{equation}\label{lm:main_lm_proj2}
\begin{split}
&\langle\nabla F(\mathbf{y}_t,\gamma_{t+1}),\mathbf{x}_{t+1}-\mathbf{y}_t \rangle = \langle\nabla F_{t+1}(\mathbf{y}_t,\gamma_{t+1})-\sigma_{t+1}(\mathbf{y}_t),\mathbf{x}_{t+1}-\mathbf{y}_t \rangle \\
&\leq -\eta_t\|\nabla F_{t+1}(\mathbf{y}_t,\gamma_{t+1})\|^2 + \eta_t \|\sigma_{t+1}(\mathbf{y}_t)\| \cdot \|\nabla F_{t+1}(\mathbf{y}_t,\gamma_{t+1})\|.
\end{split}
\end{equation}
Inserting (\ref{eq:main_lm_v_t1},\ref{eq:main_lm_aux_ineq},\ref{lm:main_lm_proj} and \ref{lm:main_lm_proj2}) into (\ref{eq:main_lm_Delta}) we have
\begin{equation}
\begin{split}
&\Delta_{t+1} \leq (1-\alpha_t)\Delta_t + (1-\alpha_t)(\gamma_t-\gamma_{t+1})D_{\mathcal{U}} +\\
&\frac{\alpha_t}{2}\left[ \theta_t\|\mathbf{x}-\mathbf{v}_t\|^2 - (\mu+\theta_t)\|\mathbf{x}-\mathbf{v}_{t+1}\|^2 \right] + \left\langle \sigma_{t+1}(\mathbf{y}_t),\alpha_t(\mathbf{x}-\mathbf{y}_t)+(1-\alpha_t)(\mathbf{x}_t-\mathbf{y}_t) \right\rangle+\\
&\eta_t\|\sigma_{t+1}(\mathbf{y}_t)\| \cdot \|\nabla F_{t+1}(\mathbf{y}_t,\gamma_{t+1})\| + \left[ \frac{\alpha_t}{2(\mu+\theta_t)} +\frac{L_{t+1}}{2}\eta_t^2 - \eta_t \right] \|\nabla F_{t+1}(\mathbf{y}_t,\gamma_{t+1})\|^2+\\
&\left\langle \nabla F_{t+1}(\mathbf{y}_t,\gamma_{t+1}), \frac{-\alpha_t\theta_t(\mathbf{v}_t-\mathbf{y}_t)}{\mu+\theta_t} -(1-\alpha_t)(\mathbf{x}_t-\mathbf{y}_t) \right\rangle.
\end{split}
\end{equation}
Taking the last term $\frac{-\alpha_t\theta_t(\mathbf{v}_t-\mathbf{y}_t)}{\mu+\theta_t} -(1-\alpha_t)(\mathbf{x}_t-\mathbf{y}_t) = 0$ recovers the updating rule of $\mathbf{y}_t$ (Line 1 of Alg.\ref{alg:ansgd}). Hence our result follows.
\end{proof}

\section{Proof of Theorem \ref{thm:result_general_convex}}
\begin{proof}
It is easy to verify that by taking $\alpha_t = \frac{2}{t+2}$, $\gamma_{t+1}=\alpha_t$ and $\theta_t = L_g\alpha_t+\frac{\mathbb{E}\|A_{\boldsymbol{\xi}}\|^2}{\zeta} + \frac{\Omega}{\sqrt{\alpha_t}}$, we have $\forall t>1$:
\begin{equation}
(1-\alpha_{t-1})(\gamma_{t-1}-\gamma_{t}) \leq \gamma_{t} - \gamma_{t+1},
\end{equation}
and
\begin{equation}
(1-\alpha_t)\frac{\alpha_{t-1}}{2(\theta_{t-1}-\alpha_{t-1}\mathbb{E}L_{t})} \leq \frac{\alpha_t}{2(\theta_t-\alpha_t\mathbb{E}L_{t+1})}.
\end{equation}

Next we define and bound weighted sums of $D_t^2$ that will be used later.
\begin{equation}\label{eq:def_psi}
\begin{split}
\Psi(t) := &\left[ \alpha_t\theta_t - (1-\alpha_t)\alpha_{t-1}\theta_{t-1}\right]D_t^2 + (1-\alpha_t)\left[\alpha_{t-1}\theta_{t-1}-(1-\alpha_{t-1})\alpha_{t-2}\theta_{t-2} \right]D_{t-1}^2 +\\
& (1-\alpha_t)(1-\alpha_{t-1})\left[ \alpha_{t-2}\theta_{t-2}-(1-\alpha_{t-2})\alpha_{t-3}\theta_{t-3} \right]D_{t-2}^2 + \cdots,
\end{split}
\end{equation}
where replacing $\alpha_t$ and $\theta_t$ by their definitions we have $\forall t$:
\begin{equation}\label{eq:bound_psi_aux}
\alpha_t\theta_t-(1-\alpha_t)\alpha_{t-1}\theta_{t-1} = \frac{4L_g}{(t+1)^2(t+2)^2} + \frac{2\mathbb{E}\|A_{\boldsymbol{\xi}}\|^2/\zeta}{(t+1)(t+2)} + \frac{\sqrt{2}\left[(t+1)\sqrt{t+2}-t\sqrt{t+1}\right]\Omega}{(t+1)(t+2)}
\end{equation}
Substituting (\ref{eq:bound_psi_aux}) into (\ref{eq:def_psi})  and using invoking the definition of $D^2$ we have $\forall t$:
\begin{equation}\label{eq:bound_psi}
\begin{split}
&\Psi(t) \leq 4L_g D^2 \left[ \frac{1}{(t+1)^2(t+2)^2} + \frac{t(t+1)}{(t+1)(t+2)}\frac{1}{t^2(t+1)^2} + \frac{(t-1)t}{(t+1)(t+2)}\frac{1}{(t-1)^2t^2} +\cdots\right]\\
&+ \frac{2\mathbb{E}\|A_{\boldsymbol{\xi}}\|^2 D^2}{\zeta} \left[ \frac{1}{(t+1)(t+2)} + \frac{t(t+1)}{(t+1)(t+2)}\frac{1}{t(t+1)} + \frac{(t-1)t}{(t+1)(t+2)}\frac{1}{(t-1)t} + \cdots \right]\\
&+ \sqrt{2}\Omega D^2 \bigg[  \frac{(t+1)\sqrt{t+2}-t\sqrt{t+1}}{(t+1)(t+2)} + \frac{t(t+1)}{(t+1)(t+2)}\frac{t\sqrt{t+1}-(t-1)\sqrt{t}}{t(t+1)} + \\
&\ \ \ \ \frac{(t-1)t}{(t+1)(t+2)}\frac{(t-1)\sqrt{t}-(t-2)\sqrt{t-1}}{(t-1)t} + \cdots\bigg]\\
&= \frac{4L_g D^2}{(t+1)(t+2)}\left[ \left(\frac{1}{t+1}-\frac{1}{t+2}\right) + \left(\frac{1}{t}-\frac{1}{t+1}\right)+ \left(\frac{1}{t-1}-\frac{1}{t}\right)+ \cdots\right]\\
& + \frac{2\mathbb{E}\|A_{\boldsymbol{\xi}}\|^2 D^2}{\zeta} \left[ \frac{1}{(t+1)(t+2)} +   \frac{1}{(t+1)(t+2)} +  \frac{1}{(t+1)(t+2)} + \cdots\right]\\
& + \frac{\sqrt{2}\Omega D^2}{(t+1)(t+2)} \left[ (t+1)\sqrt{t+2} - t\sqrt{t+1} + t\sqrt{t+1} - (t-1)\sqrt{t} + (t-1)\sqrt{t} - (t-2)\sqrt{t-1} + \cdots\right]\\
&\leq \alpha_t\theta_t D^2.
\end{split}
\end{equation}

Since $\mu=0$, by recursively applying (\ref{eq:result_base}) and $1-\alpha_0 = 0$ we have
\begin{equation}
\begin{split}
\mathbb{E}\Delta_{t+1} &\leq (1-\alpha_t)\mathbb{E}\Delta_t + \alpha_t\theta_t\left(D_t^2-D_{t+1}^2\right) + \frac{\alpha_t}{2(\theta_t-\alpha_t\mathbb{E}L_{t+1})}
\sigma^2 + (1-\alpha_t)(\gamma_t-\gamma_{t+1})D_{\mathcal{U}}\\
&\leq(1-\alpha_t)(1-\alpha_{t-1})\mathbb{E}\Delta_{t-1} + \alpha_t\theta_t\left(D_t^2-D_{t+1}^2\right) +
\left(1-\alpha_t\right)\alpha_{t-1}\theta_{t-1}\left(D_{t-1}^2-D_{t}^2\right) + \\
&\ \ \ \ \frac{2\alpha_t}{2(\theta_t-\alpha_t\mathbb{E}L_{t+1})}
\sigma^2 + 2(1-\alpha_t)(\gamma_t-\gamma_{t+1})D_{\mathcal{U}}\\
&\leq \cdots\\
&\overset{(\ref{eq:def_psi})}{\leq} \prod_{i=0}^t (1-\alpha_i)\Delta_0 + \Psi(t) + \frac{(t+1)\alpha_t}{2(\theta_t-\alpha_t\mathbb{E}L_{t+1})}
\sigma^2 + (t+1)(1-\alpha_t)(\gamma_t-\gamma_{t+1})D_{\mathcal{U}}\\
&\overset{(\ref{eq:bound_psi})}{\leq} \alpha_t\theta_t D^2 + \frac{\sigma^2}{\theta_t-\alpha_t\mathbb{E}L_{t+1}} + \frac{2D_{\mathcal{U}}}{t+2}\\
&= \left[ \alpha_t^2\mathbb{E}L_{t+1} + \Omega\sqrt{\alpha_t} \right] D^2 + \frac{\sqrt{\alpha_t} \sigma^2}{\Omega} + \frac{2D_{\mathcal{U}}}{t+2}.
\end{split}
\end{equation}
Combining with Lemma \ref{lm:agd_phi} we have $\forall \mathbf{x}$
\begin{equation}
\begin{split}
\mathbb{E}\left[\Phi(\mathbf{x}_{t+1}) - \Phi(\mathbf{x}) \right] &\leq \left[ \alpha_t^2\mathbb{E}L_{t+1} + \Omega\sqrt{\alpha_t} \right] D^2 + \frac{\sqrt{\alpha_t} \sigma^2}{\Omega} + \frac{2D_{\mathcal{U}}}{t+2} + \gamma_{t+1} D_{\mathcal{U}}\\
&\leq \alpha_t^2 L_g D^2 + \left(\gamma_{t+1}+\frac{2}{t+2}\right) D_{\mathcal{U}} + \alpha_t^2 \frac{\mathbb{E}\|A_{\boldsymbol{\xi}}\|^2}{\gamma_{t+1}\zeta} D^2 + \sqrt{\alpha_t} \left(\Omega D^2+\frac{\sigma^2}{\Omega}\right).
\end{split}
\end{equation}
Taking $\gamma_{t+1} = \alpha_t = \frac{2}{t+2}$ our result follows.
\end{proof}

\section{Proof of Theorem \ref{thm:result_strongly_convex}}
\begin{proof}
It is easy to verify that by taking $\alpha_t=\frac{2}{t+1}$, we have $\forall t\geq 1$
\begin{equation}\label{eq:th2pf_alpha2}
(1-\alpha_{t-1})(\gamma_{t-1}-\gamma_{t}) \leq \gamma_{t} - \gamma_{t+1}.
\end{equation}
and
\begin{equation}\label{eq:th2pf_alpha}
(1-\alpha_t)\alpha_{t-1}^2 \leq \alpha_t^2
\end{equation}
Denote
\begin{equation}
S_t := \alpha_t\theta_t - (1-\alpha_t)(\alpha_{t-1}\theta_{t-1} + \mu\alpha_{t-1}).
\end{equation}
Taking $\theta_t = L_g\alpha_t + \frac{\mu}{2\alpha_t}+\frac{\mathbb{E}\|A_{\boldsymbol{\xi}}\|^2}{\zeta} - \mu$ it is easy to verify that $\forall t\geq 1$:
\begin{equation}\label{eq:th2pf_theta}
S_t= 4L_g\frac{1}{(t+1)^2 t^2} + \frac{2\mathbb{E}\|A_{\boldsymbol{\xi}}\|^2}{\zeta}\left[ \frac{1}{t}-\frac{1}{t+1} \right] - \frac{\mu}{t+1}.
\end{equation}
We want to find the smallest iteration index $C$ such that: when $t\geq C$, $S_t \leq 0$. Without any knowledge about $L_g$ and $\mathbb{E}\|A_{\boldsymbol{\xi}}\|^2$, minimizing $S_t$ w.r.t $t$ does not yield an analytic form of $C$. Hence we simply let
\begin{equation}\label{eq:th2_ct_1}
4L_g\frac{1}{(t+1)^2 t^2} \leq \frac{\mu}{2(t+1)} ,
\end{equation}
and
\begin{equation}\label{eq:th2_ct_2}
\frac{2\mathbb{E}\|A_{\boldsymbol{\xi}}\|^2}{\zeta}\left[ \frac{1}{t}-\frac{1}{t+1} \right] \leq \frac{\mu}{2(t+1)}.
\end{equation}
Inequality (\ref{eq:th2_ct_1}) is satisfied when
\begin{equation}
t\geq 2\left( \frac{L_g}{\mu}\right)^{1/3},
\end{equation}
and (\ref{eq:th2_ct_2}) is satisfied when
\begin{equation}
t\geq \frac{4\mathbb{E}\|A_{\boldsymbol{\xi}}\|^2}{\zeta\mu}.
\end{equation}
Combining these two we reach the definition of $C$ in (\ref{eq:thm2_def_c}). Next we proceed to prove the bound.

As defined in the theorem, we denote $\tilde{D}^2=\max_{0\leq i\leq{\min(t,C)}} D_i^2$. By recursively applying (\ref{eq:result_base}) for $0\leq i\leq t$ and noticing that $S_t\leq 0\ \forall t\geq C$, $1-\alpha_1 = 0$ we have
\begin{equation}
\begin{split}
\mathbb{E}\Delta_{t+1} &\overset{(\ref{eq:th2pf_alpha2})}{\leq} \prod_{i=0}^t (1-\alpha_i)\Delta_0 + (t+1)(1-\alpha_t)(\gamma_t-\gamma_{t+1})D_{\mathcal{U}} +\\
&\ \ \ \ \left[(\alpha_t\theta_t)D_t^2 -(\alpha_t\theta_t+\mu\alpha_t)D_{t+1}^2  \right] +\\
&\ \ \ \ (1-\alpha_t)\left[(\alpha_{t-1}\theta_{t-1})D_{t-1}^2 -(\alpha_{t-1}\theta_{t-1}+\mu\alpha_{t-1})D_{t}^2 \right] +\\
&\ \ \ \ (1-\alpha_t)(1-\alpha_{t-1})\left[(\alpha_{t-2}\theta_{t-2})D_{t-2}^2 -(\alpha_{t-2}\theta_{t-2}+\mu\alpha_{t-2})D_{t-1}^2 \right] +\\
&\ \ \ \ \cdots + \prod_{i=1}^t (1-\alpha_i)\left[ (\alpha_{0}\theta_{0})D_{0}^2 -(\alpha_{0}\theta_{0}+\mu\alpha_{0})D_{1}^2 \right] + \\
&\ \ \ \ \frac{\sigma^2}{\mu}\left[\alpha_t^2 + (1-\alpha_{t})\alpha_{t-1}^2 + \cdots + \prod_{i=1}^t (1-\alpha_i) \alpha_0^2 \right]\\
& \overset{(\ref{eq:th2pf_alpha})}{\leq} \frac{2D_{\mathcal{U}}}{t+1} + \tilde{D}^2 \prod_{i=C-1}^t(1-\alpha_i)\left[ \alpha_{C-2}\theta_{C-2}-(1-\alpha_{C-2})(\alpha_{C-3}\theta_{C-3}+\mu\alpha_{C-3})\right]+\\
&\ \ \ \ \ \ \tilde{D}^2 \prod_{i=C-2}^t(1-\alpha_i)\left[ \alpha_{C-3}\theta_{C-3}-(1-\alpha_{C-3})(\alpha_{C-4}\theta_{C-4}+\mu\alpha_{C-4})\right]+\\
&\ \ \ \ \ \ \cdots + \tilde{D}^2 \prod_{i=2}^t (1-\alpha_i)\left[ \alpha_1\theta_1-(1-\alpha_1)(\alpha_{0}\theta_{0}+\mu\alpha_{0})\right]+
\frac{t\alpha_t^2\sigma^2}{\mu}
\end{split}
\end{equation}
Applying (\ref{eq:th2pf_theta}) by ignoring the $-\frac{\mu}{t+1}$ term to the above inequality we can bound the coefficients of $L_g$ and $\frac{\mathbb{E}\|A_{\boldsymbol{\xi}}\|^2}{\zeta}$ parts separately as follows.

When $t\geq C$, for the $L_g$ part:
\begin{equation}
\begin{split}
&\frac{\prod_{i=C-1}^t(1-\alpha_i)}{(C-1)^2 (C-2)^2} + \frac{\prod_{i=C-2}^t(1-\alpha_i)}{(C-2)^2 (C-3)^2} + \frac{\prod_{i=C-3}^t(1-\alpha_i)}{(C-3)^2 (C-4)^2} + \cdots + \frac{\prod_{i=2}^t(1-\alpha_i)}{2^2 \cdot1^2}\\
&= \frac{1}{(t+1)t} \left[\frac{1}{(C+2)(C+1)}+ \frac{1}{(C+1)C} + \frac{1}{C(C-1))} + \cdots + \frac{1}{2\cdot 1} \right]\\
&\leq \frac{1}{(t+1)t} \sum_{i=1}^{C+1}\frac{1}{i^2} \leq \frac{\pi^2}{6t(t+1)}
\end{split}
\end{equation}
For the $\frac{\mathbb{E}\|A_{\boldsymbol{\xi}}\|^2}{\zeta}$ part:
\begin{equation}
\begin{split}
&\Pi_{i=C-1}^t(1-\alpha_i)\left(\frac{1}{C-2}-\frac{1}{C-1} \right) + \Pi_{i=C-2}^t(1-\alpha_i)\left(\frac{1}{C-3}-\frac{1}{C-2} \right)+ \cdots + \prod_{i=2}^t(1-\alpha_i) \left( 1-\frac{1}{2} \right)\\
& = \frac{C-1}{(t+1)t}-\frac{C-2}{(t+1)t} + \frac{C-2}{(t+1)t}-\frac{C-3}{(t+1)t}+ \cdots + \frac{2}{(t+1)t} - \frac{1}{(t+1)t}\\
& = \frac{C-1}{(t+1)t}-\frac{1}{(t+1)t} = \frac{C-2}{t(t+1)}.
\end{split}
\end{equation}
Combining with Lemma \ref{lm:agd_phi} and taking $\gamma_{t+1} = \alpha_t = \frac{2}{t+1}$ we have $\forall \mathbf{x}$:
\begin{equation}
\begin{split}
\mathbb{E}\left[\Phi(\mathbf{x}_{t+1}) - \Phi(\mathbf{x}) \right] &\leq \frac{2D_{\mathcal{U}}}{t+1}+ \frac{2\pi^2L_g \tilde{D}^2}{3t(t+1)} +\frac{2(C-2)\mathbb{E}\|A_{\boldsymbol{\xi}}\|^2 \tilde{D}^2/\zeta }{t(t+1)} +\frac{\sigma^2}{\mu(t+1)} + \gamma_{t+1}D_{\mathcal{U}}\\
&=\frac{2\pi^2L_g \tilde{D}^2}{3t(t+1)} +\frac{2(C-2)\mathbb{E}\|A_{\boldsymbol{\xi}}\|^2 \tilde{D}^2/\zeta }{t(t+1)}+\frac{4D_{\mathcal{U}}}{t+1} +\frac{\sigma^2}{\mu(t+1)}.
\end{split}
\end{equation}
When $0\leq t\leq C$, one can simply put $C=t$ in the above, and this completes our proof.
\end{proof}

\section{Proof of Proposition \ref{prop:bto}}
\begin{proof}
\begin{equation*}
\begin{split}
&\mathbb{E}_{\boldsymbol{\xi}_{[t]}} R(t)=\mathbb{E}_{\boldsymbol{\xi}_{[t]}} \sum_{i=0}^{t-1} \left[\Phi(\mathbf{x}_{i}, \boldsymbol{\xi}_{i+1}) - \Phi(\mathbf{x}_{t}^*, \boldsymbol{\xi}_{i+1})\right]\\
&=\mathbb{E}_{\boldsymbol{\xi}_{[t]}} \sum_{i=0}^{t-1}  \bigg\{ \left[\Phi(\mathbf{x}_{i}, \boldsymbol{\xi}_{i+1})-\Phi(\mathbf{x}^*)\right] +\left[\Phi(\mathbf{x}^*)-\Phi(\mathbf{x}_{t}^*, \boldsymbol{\xi}_{i+1})\right] \bigg\}\\
&= \sum_{i=0}^{t-1}\mathbb{E}_{\boldsymbol{\xi}_{[i+1]}} \left[\Phi(\mathbf{x}_{i}, \boldsymbol{\xi}_{i+1})-\Phi(\mathbf{x}^*)\right] + \mathbb{E}_{\boldsymbol{\xi}_{[t]}} \sum_{i=0}^{t-1}\left[ \Phi(\mathbf{x}^*)-\Phi(\mathbf{x}_t^*) \right] + \mathbb{E}_{\boldsymbol{\xi}_{[t]}} \sum_{i=0}^{t-1}\left[ \Phi(\mathbf{x}_t^*)-\Phi(\mathbf{x}_{t}^*, \boldsymbol{\xi}_{i+1}) \right]\\
&\leq \sum_{i=0}^{t-1}\mathbb{E}_{\boldsymbol{\xi}_{[i+1]}} \left[\Phi(\mathbf{x}_{i}, \boldsymbol{\xi}_{i+1})-\Phi(\mathbf{x}^*)\right] + \mathbb{E}_{\boldsymbol{\xi}_{[t]}} \sum_{i=0}^{t-1}\left[ \Phi(\mathbf{x}_t^*)-\Phi(\mathbf{x}_{t}^*, \boldsymbol{\xi}_{i+1}) \right]\\
&= \sum_{i=0}^{t-1} \mathbb{E}_{\boldsymbol{\xi}_{[i]}}\left[\Phi(\mathbf{x}_i) - \Phi(\mathbf{x}^*) \right] + \mathbb{E}_{\boldsymbol{\xi}_{[t]}}\sum_{i=0}^{t-1}\left[\Phi(\mathbf{x}_t^*)-\Phi(\mathbf{x}_t^*,\boldsymbol{\xi}_{i+1}) \right].
\end{split}
\end{equation*}
\end{proof}

\bibliography{ANSGD_JMLR}

\begin{thebibliography}{32}
\providecommand{\natexlab}[1]{#1}
\providecommand{\url}[1]{\texttt{#1}}
\expandafter\ifx\csname urlstyle\endcsname\relax
  \providecommand{\doi}[1]{doi: #1}\else
  \providecommand{\doi}{doi: \begingroup \urlstyle{rm}\Url}\fi

\bibitem[Agarwal et~al.(2012)Agarwal, Bartlett, Ravikumar, and
  Wainwright]{agarwal12itlb}
Alekh Agarwal, Peter~L. Bartlett, P.~Ravikumar, and Martin~J. Wainwright.
\newblock Information-theoretic lower bounds on the oracle complexity of
  stochastic convex optimization.
\newblock \emph{Information Theory, IEEE Trans.}, 2012.

\bibitem[Bach and Moulines(2011)]{bach11naasaaml}
Francis Bach and Eric Moulines.
\newblock Non-asymptotic analysis of stochastic approximation algorithms for
  machine learning.
\newblock In \emph{NIPS}, 2011.

\bibitem[Beck and Teboulle(2009)]{beck09fista}
Amir Beck and Marc Teboulle.
\newblock A fast iterative shrinkage-thresholding algorithm for linear inverse
  problems.
\newblock \emph{SIAM J. Imaging Sci.}, 2\penalty0 (1):\penalty0 193--202, 2009.

\bibitem[Bottou()]{bottousgd2}
Leon Bottou.
\newblock Stochastic gradient descent 2.0.
\newblock URL \url{http://leon.bottou.org/projects/sgd}.

\bibitem[Boucheron et~al.(2005)Boucheron, Bousquet, and
  Lugosi]{boucheron05tcassra}
St{\'{e}}phane Boucheron, Olivier Bousquet, and G{\'{a}}bor Lugosi.
\newblock Theory of classification: A survey of some recent advances.
\newblock \emph{ESAIM: Probability and Statistics}, 9:\penalty0 323--375, 2005.

\bibitem[Chen and Teboulle(1993)]{chen93threepoint}
Gong Chen and Marc Teboulle.
\newblock Convergence analysis of a proximal-like minimization algorithm using
  bregman functions.
\newblock \emph{SIAM J. on Optimization}, 3\penalty0 (3), 1993.

\bibitem[Daubechies et~al.(2004)Daubechies, Defrise, and Mol]{daubechies04ita}
I.~Daubechies, M.~Defrise, and C.~De Mol.
\newblock An iterative thresholding algorithm for linear inverse problems with
  a sparsity constraint.
\newblock \emph{Communications on Pure and Applied Mathematics}, 57\penalty0
  (11):\penalty0 1413–1457, 2004.

\bibitem[Dekel et~al.(2010)Dekel, Gilad-Bachrach, Shamir, and
  Xiao]{dekel10odopmb}
Ofer Dekel, Ran Gilad-Bachrach, Ohad Shamir, and Lin Xiao.
\newblock Optimal distributed online prediction using mini-batches.
\newblock \emph{arXiv}, 2010.
\newblock URL \url{http://arxiv.org/abs/1012.1367}.

\bibitem[Duchi and Singer(2009)]{duchi09fobos}
John Duchi and Yoram Singer.
\newblock Efficient online and batch learning using forward backward splitting.
\newblock \emph{JMLR}, \penalty0 (10):\penalty0 2899--2934, 2009.

\bibitem[Duchi et~al.(2011)Duchi, Bartlett, and Wainwright]{duchi11rsso}
John Duchi, Peter~L. Bartlett, and Martin~J. Wainwright.
\newblock Randomized smoothing for stochastic optimization.
\newblock \emph{arXiv}, 2011.
\newblock URL \url{http://arxiv.org/abs/1103.4296}.

\bibitem[Hastie et~al.(2009)Hastie, Tibshirani, and Friedman]{hastie09esl}
Trevor Hastie, Robert Tibshirani, and Jerome Friedman.
\newblock \emph{The Elements of Statistical Learning: Data Mining, Inference,
  and Prediction}.
\newblock Springer, 2nd edition, 2009.

\bibitem[Hazan and Kale(2011)]{hazan11brmb}
Elad Hazan and Satyen Kale.
\newblock Beyond the regret minimization barrier: an optimal algorithm for
  stochastic strongly-convex optimization.
\newblock In \emph{COLT}, 2011.

\bibitem[Hu et~al.(2009)Hu, Kwok, and Pan]{hu09agmsool}
Chonghai Hu, James~T. Kwok, and Weike Pan.
\newblock Accelerated gradient methods for stochastic optimization and online
  learning.
\newblock In \emph{NIPS 22}, 2009.

\bibitem[Huber(1964)]{huber64relp}
Peter~J. Huber.
\newblock Robust estimation of a location parameter.
\newblock \emph{Annals of Mathematical Statistics}, 35\penalty0 (1):\penalty0
  73--101, 1964.

\bibitem[Lan and Ghadimi(2011)]{lan11osaascsco1}
G.~Lan and S.~Ghadimi.
\newblock Optimal stochastic approximation algorithms for strongly convex
  stochastic composite optimization, i: a generic algorithmic framework.
\newblock \emph{SIAM J. on Optimization}, 2011.

\bibitem[Lan(2010)]{lan10omsco}
Guanghui Lan.
\newblock An optimal method for stochastic composite optimization.
\newblock \emph{Mathematical Programming}, 2010.
\newblock \doi{DOI 10.1007/s10107-010-0434-y}.

\bibitem[Lions and Mercier(1979)]{lions79splitting}
P.~L. Lions and B.~Mercier.
\newblock Splitting algorithms for the sum of two nonlinear operators.
\newblock \emph{SIAM J. on Numerical Analysis}, 16\penalty0 (6):\penalty0
  964--979, 1979.

\bibitem[Nemirovski and Yudin(1983)]{nemirovski83pcmeo}
A.~Nemirovski and D.~Yudin.
\newblock \emph{Problem Complexity and Method Efficiency in Optimization}.
\newblock John Wiley and Sons, 1983.

\bibitem[Nemirovski et~al.(2009)Nemirovski, Juditsky, Lan, and
  Shapiro]{nemirovski09rsaasp}
A.~Nemirovski, A.~Juditsky, G.~Lan, and A.~Shapiro.
\newblock Robust stochastic approximation approach to stochastic programming.
\newblock \emph{SIAM J. on Optimization}, 19\penalty0 (4):\penalty0 1574--1609,
  2009.

\bibitem[Nesterov(2004)]{nesterov04ilco}
Yurii Nesterov.
\newblock \emph{Introductory Lectures on Convex Optimization, A Basic Course}.
\newblock Kluwer Academic Publishers, 2004.

\bibitem[Nesterov(2005{\natexlab{a}})]{nesterov05egtncm}
Yurii Nesterov.
\newblock Excessive gap technique in nonsmooth convex minimization.
\newblock \emph{SIAM J. Optim.}, 16\penalty0 (1):\penalty0 235--249,
  2005{\natexlab{a}}.

\bibitem[Nesterov(2005{\natexlab{b}})]{nesterov05smnsf}
Yurii Nesterov.
\newblock Smooth minimization of non-smooth functions.
\newblock \emph{Math. Program., Ser. A}, 103:\penalty0 127--152,
  2005{\natexlab{b}}.

\bibitem[Nesterov(2007{\natexlab{a}})]{nesterov07cof}
Yurii Nesterov.
\newblock Gradient methods for minimizing composite objective function.
\newblock Technical Report CORE DISCUSSION PAPER 2007/76, Catholic University
  of Louvain, 2007{\natexlab{a}}.

\bibitem[Nesterov(2007{\natexlab{b}})]{nesterov07staso}
Yurii Nesterov.
\newblock Smoothing technique and its applications in semidefinite
  optimization.
\newblock \emph{Mathematical Programming}, 110\penalty0 (2):\penalty0 245--259,
  2007{\natexlab{b}}.

\bibitem[Polyak and Juditsky(1992)]{polyak92asaa}
Boris~T. Polyak and Anatoli~B. Juditsky.
\newblock Acceleration of stochastic approximation by averaging.
\newblock \emph{SIAM J. on Control and Optimization}, 30\penalty0 (4):\penalty0
  838--855, 1992.

\bibitem[Robbins and Monro(1951)]{robbins51sam}
Herbert Robbins and Sutton Monro.
\newblock A stochastic approximation method.
\newblock \emph{The Annals of Mathematical Statistics}, 22\penalty0
  (3):\penalty0 400--407, 1951.

\bibitem[Shalev-Shwartz et~al.(2007)Shalev-Shwartz, Singer, and
  Srebro]{sss07pegasos}
Shai Shalev-Shwartz, Yoram Singer, and Nathan Srebro.
\newblock Pegasos: Primal estimated sub-gradient solver for svm.
\newblock In \emph{ICML}, 2007.

\bibitem[Shalev-Shwartz et~al.(2009)Shalev-Shwartz, Shamir, Srebro, and
  Sridharan]{sss09sco}
Shai Shalev-Shwartz, Ohad Shamir, Nathan Srebro, and Karthik Sridharan.
\newblock Stochastic convex optimization.
\newblock In \emph{COLT}, 2009.

\bibitem[Shamir(2011)]{shamir11mgdoscso}
Ohad Shamir.
\newblock Making gradient descent optimal for strongly convex stochastic
  optimization.
\newblock In \emph{OPT 2011}, 2011.
\newblock URL \url{http://arxiv.org/abs/1109.5647}.

\bibitem[Wright et~al.(2009)Wright, Nowak, and Figueiredo]{wright09srsa}
S.~J. Wright, R.~D. Nowak, and M.~A.~T. Figueiredo.
\newblock Sparse reconstruction by separable approximation.
\newblock \emph{IEEE Transactions on Signal Processing}, 57\penalty0
  (7):\penalty0 2479--2493, 2009.

\bibitem[Xiao(2010)]{xiao10damrsloo}
Lin Xiao.
\newblock Dual averaging methods for regularized stochastic learning and online
  optimization.
\newblock \emph{JMLR}, 11:\penalty0 2543--2596, 2010.

\bibitem[Xu(2011)]{xu11tooplslasgd}
Wei Xu.
\newblock Towards optimal one pass large scale learning with averaged
  stochastic gradient descent.
\newblock \emph{arXiv}, 2011.
\newblock URL \url{http://arxiv.org/abs/1107.2490}.

\end{thebibliography}

\end{document}